\documentclass[letterpaper]{article} 
\usepackage{aaai25}  
\usepackage{times}  
\usepackage{helvet}  
\usepackage{courier}  
\usepackage[hyphens]{url}  
\usepackage{graphicx} 
\usepackage{natbib}  
\usepackage{caption} 
\usepackage{algorithm}
\usepackage{algorithmic}
\usepackage{booktabs} 
\usepackage{multirow}
\usepackage{amsmath}
\usepackage{amssymb}
\usepackage{mathtools}
\usepackage{amsthm}
\newtheorem*{corollary}{Corollary}

\newtheorem{theorem}{Theorem}

\newtheorem{proposition}{Proposition}

\usepackage{xcolor}
\usepackage{color, colortbl}
\usepackage{mwe}

\usepackage{algorithm,algorithmic}

\usepackage{multicol}
\usepackage{multirow}

\usepackage[capitalize,noabbrev]{cleveref}

\usepackage{newfloat}
\usepackage{listings}
\DeclareCaptionStyle{ruled}{labelfont=normalfont,labelsep=colon,strut=off} 
\floatstyle{ruled}
\newfloat{listing}{tb}{lst}{}
\floatname{listing}{Listing}
\title{Enhancing Diffusion Models for Inverse Problems with Covariance-Aware Posterior Sampling}
\author{
    Shayan Mohajer Hamidi\\
En-Hui Yang
}
\affiliations{
    Department of Electrical and Computer Engineering, University of Waterloo \\

    200 University Ave W\\
    Waterloo, ON N2L 3G1\\
}

\usepackage{bibentry}

\begin{document}

\maketitle

\begin{abstract}
Inverse problems exist in many disciplines of  science and engineering. In computer vision, for example,  tasks such as inpainting, deblurring, and super-resolution can be effectively modeled as inverse problems. Recently, denoising diffusion probabilistic models (DDPMs) are shown to provide a promising solution to noisy linear inverse problems without the need for additional task-specific training. Specifically, with the prior provided by DDPMs, one can sample from the posterior by approximating the likelihood. In the literature, approximations of the likelihood are often based on the mean of conditional densities of the reverse process, which can be obtained using Tweedie’s formula. To obtain a better approximation to the likelihood, in this paper we first derive a closed-form formula for the covariance of the reverse process. Then, we propose a method based on finite difference method to approximate this covariance such that it can be readily obtained from the existing pre-trained DDPMs, thereby not increasing the complexity compared to existing approaches. Finally, based on the mean and approximated covariance of the reverse process, we present a new approximation to the likelihood. We refer to this method as \textbf{c}ovariance-\textbf{a}ware \textbf{d}iffusion \textbf{p}osterior \textbf{s}ampling (CA-DPS). Experimental results show that CA-DPS significantly improves reconstruction performance without requiring hyperparameter tuning. The code for the paper is put in the supplementary materials.
\end{abstract}

\section{Introduction} \label{sec:intro}
Denoising diffusion probabilistic models (DDPMs) \citep{ho2020denoising} have made remarkable advancements in data synthesis over the past few years, revolutionizing fields such as image synthesis \citep{nichol2022glide,saharia2022photorealistic,zhang2023text}, video generation \citep{ho2022video} and audio synthesis \citep{kongdiffwave}.

Given the powerful ability of DDPMs to estimate target distributions, one promising application is to use them to solve linear inverse problems such as denoising, inpainting, deblurring, and super-resolution. These tasks aim to recover a signal $\boldsymbol{x}_0$ (e.g., a face image) from a measurement $\boldsymbol{y}$, where $\boldsymbol{y}$ is related to $\boldsymbol{x}_0$ through the forward measurement operator $\bold{A}$ and detector noise $\boldsymbol{n}$ \citep{song2020denoising,chungdiffusion, song2023pseudoinverse, dou2024diffusion,peng2024improving}. A naive approach to using DDPMs for solving inverse problems is to train a conditional DDPM to estimate the posterior $p(\boldsymbol{x}_0 | \boldsymbol{y})$ through  supervised learning. However, this approach can be computationally demanding, as it requires training separate models for different measurement operators.

To tackle the issue mentioned above, a newer method to approximate the posterior seeks to leverage pre-trained unconditional DDPMs that estimate the prior $p(\boldsymbol{x}_0)$, thereby avoiding the need for additional training. In this approach, the prior $p(\boldsymbol{x}_0)$ obtained from DDPMs is combined with the likelihood $p(\boldsymbol{y} | \boldsymbol{x}_0)$ to sample from the posterior distribution for inverse problems. However, because the likelihood term $p(\boldsymbol{y} | \boldsymbol{x}_0)$ is analytically intractable in the context of DDPMs due to their time-dependent nature, it must be approximated in some way \citep{chungdiffusion}.

To approximate the likelihood $p(\boldsymbol{y} | \boldsymbol{x}_0)$, there are mainly two approaches in the literature as we discuss in the sequel. The first approach relies on projections onto the measurement subspace \citep{song2020score,chung2022come,choi2021ilvr}. However, these projection-based methods perform poorly in the presence of noise in the measurements, as the noise tends to be amplified during the generative process due to the ill-posed nature of inverse problems \citep{chungdiffusion}. The second approach leverages the relationship  $p(\boldsymbol{y} | \boldsymbol{x}_t) = \int p(\boldsymbol{y} | \boldsymbol{x}_0) p(\boldsymbol{x}_0 | \boldsymbol{x}_t) d \boldsymbol{x}_0$ in DDPMs; as such,  assuming that $p(\boldsymbol{y} | \boldsymbol{x}_0) $ is known, one can approximate $p(\boldsymbol{y} | \boldsymbol{x}_t)$ by estimating $p(\boldsymbol{x}_0 | \boldsymbol{x}_t)$. Although the distribution of $p(\boldsymbol{x}_0 | \boldsymbol{x}_t)$ is still intractable, the conditional mean $\Tilde{\boldsymbol{x}}_0 = \mathbb{E} (\boldsymbol{x}_0 | \boldsymbol{x}_t)$ can be analytically obtained using Tweedie’s formula \cite{efron2011tweedie}. The conditional mean $\Tilde{\boldsymbol{x}}_0$ is then used by \citet{chungdiffusion} to approximate $p(\boldsymbol{x}_0 | \boldsymbol{x}_t)$ by delta distribution $\delta (\boldsymbol{x}_0 - \Tilde{\boldsymbol{x}}_0)$, and used by \citet{song2023pseudoinverse} to approximate $p(\boldsymbol{x}_0 | \boldsymbol{x}_t)$ with a Gaussian distribution $\mathcal{N}(\Tilde{\boldsymbol{x}}_0, r^2_t \bold{I})$ with a heuristically selected variance $r^2_t$.  

Nonetheless, approximating $p(\boldsymbol{x}_0 | \boldsymbol{x}_t)$ using only its first moment (mean) is prone to sub-optimal performance due to biases in reconstruction \citep{jalal2021fairness,meng2021estimating}. As a remedy, this paper aims to improve the approximation of $p(\boldsymbol{x}_0 | \boldsymbol{x}_t)$ by incorporating its second moment. Particularly, we derive a closed-form expression for the conditional covariance $\text{Cov}(\boldsymbol{x}_0 | \boldsymbol{x}_t)$ in DDPMs, and show that it depends on the Hessian $\bold{H}_t = \nabla^2_{\boldsymbol{x}_t}\log p_t(\boldsymbol{y}|\boldsymbol{x}_t)$. Yet, 
the Hessian $\bold{H}_t$ is not directly available for DDPMs, as these models only provide the score function $\nabla_{\boldsymbol{x}_t}\log p_t(\boldsymbol{y}|\boldsymbol{x}_t)$. To address this, we approximate the Hessian $ \bold{H}_t $ with a diagonal matrix $ \tilde{\bold{H}}_t $, where the diagonal elements are derived from the gradient vector $ \nabla_{\boldsymbol{x}_t} \log p_t(\boldsymbol{y} |\boldsymbol{x}_t) $ using the finite difference method. In this regard, this approximation can be easily obtained from existing pre-trained DDPMs, thus avoiding any additional complexities. Then, using $\Tilde{\bold{H}}_t$ we obtain an approximation to the true $\text{Cov}(\boldsymbol{x}_0 | \boldsymbol{x}_t)$, which we denote by $\Tilde{\boldsymbol{\Sigma}}_t$. Henceforth, we approximate $p(\boldsymbol{x}_0 | \boldsymbol{x}_t)$ with $\mathcal{N}(\Tilde{\boldsymbol{x}}_0, \Tilde{\boldsymbol{\Sigma}}_t)$. We refer to this method of deploying DDPMs for solving inverse probelms as \textbf{c}ovariance-\textbf{a}ware \textbf{d}iffusion \textbf{p}osterior \textbf{s}ampling (CA-DPS) hereafter.

It is worth noting that while some prior work has investigated using second-order approximations for the posterior, these approaches either (i) rely on the availability of second-order scores or the Jacobian of the first-order score from the diffusion model \citep{boys2023tweedie}, or (ii) require retraining existing unconditional diffusion models to output both posterior mean and variance, which increases time and memory complexity \citep{peng2024improving} (see \cref{sec:higher} for further discussion). 

The contributions of the paper are summarized as follows:

\noindent $\bullet$ For a  general exponential conditional distribution family, we derive a closed-form expression for its posterior covariance (see \cref{theorem:var}).

\noindent $\bullet$ Using \cref{theorem:var}, we determine a closed-form formula for the conditional covariance $\text{Cov}(\boldsymbol{x}_0 | \boldsymbol{x}_t)$ in DDPMs (see Corollary 1). Based on this closed-form formula,  we then introduce a method based on the finite difference approach to approximate the conditional covariance $\text{Cov}(\boldsymbol{x}_0 | \boldsymbol{x}_t)$, allowing it to be readily obtained from existing pre-trained DDPMs.

\noindent $\bullet$  By conducting experiments on two popular datasets FFHQ \cite{karras2019style} and ImageNet \citep{deng2009imagenet}, we show that CA-DPS outperforms existing approaches across various tasks, including inpainting, deblurring, and super-resolution, while also eliminating the need for hyperparameter tuning.

\section{Related Work}
\label{sec:related}

\subsection{Diffusion Models for Inverse Problems}
The use of diffusion models for solving inverse problems by sampling from the posterior has recently gained significant traction in various fields, including image denoising \citep{kawar2022denoising}, compressed sensing \citep{bora2017compressed, kadkhodaie2021stochastic}, magnetic resonance imaging (MRI) \citep{jalal2021robust}, projecting score-based stochastic differential equations (SDEs) \citep{song2022solving}, and variational approaches \citep{mardani2023variational, feng2023efficient}. In particular, the most relevant line of work, which we will review in detail in \cref{sec:setup}, involves using Tweedie's formula \citep{efron2011tweedie} to approximate the smoothed likelihood, as deployed in methods like diffusion posterior sampling (DPS) \citep{chungdiffusion} and pseudo-guided diffusion models ($\Pi$GDM) \citep{song2023pseudoinverse}. Similar strategies are also employed using singular-value decomposition (SVD) based approaches \citep{kawar2021snips}.

\subsection{Higher Order Approximation of reverse process} \label{sec:higher}
The approach presented in this paper can be seen as a variant of high-order denoising score matching~\citep{meng2021estimating, lu2022maximum}, which aims to train a diffusion model capable of learning the higher-order moments of the reverse process. However, these methods are typically limited to small-scale datasets due to their computational complexity.

Similarly to our work, \citet{boys2023tweedie} aims to estimate the covariance of the reverse process. However, their method require that the second-order scores or the Jacobian of
the first-order score be available by the diffusion model. In addition, \citet{peng2024improving} proposed a method to optimize the posterior likelihood. They proposed two methods for (i) when reverse covariance prediction is
available from the given unconditional diffusion model, and (ii) when reverse covariance prediction is not available. Their first approach is different from our proposed method as our method do not require reverse covariance to be available. Additionally, their second approach is based on Monte Carlo estimation which incurs extra complexity to the sampling process. We also acknowledge the work of \citet{stevens2023removing}, who explored a maximum-a-posteriori approach to estimate the moments of the posterior.

\section{Background and Preliminaries} \label{sec:pre}

\subsection{Diffusion Models}
Diffusion models characterize a generative process as the reverse of a noise addition process. In particular, \cite{song2020score} introduced the It$\hat{\rm{o}}$ stochastic differential equation (SDE) to describe the noise addition process (i.e., the forward SDE) for the data $\boldsymbol{x}(t)$ over the time interval $t \in [0, T]$, where $\boldsymbol{x}(t) \in \mathbb{R}^d$ for all $t$. 

In this paper, we adopt the variance-preserving form of the SDE (VP-SDE) \citep{song2020score}, which is equivalent to the DDPM framework \citep{ho2020denoising} whose equation is given as follows 
\begin{align}
\label{eq:forward-sde}
    d\boldsymbol{x} = -\frac{\beta(t)}{2} \boldsymbol{x} \, dt + \sqrt{\beta(t)}\, d\boldsymbol{w},
\end{align}
where $\beta(t): \mathbb{R} \rightarrow \mathbb{R}^{+}$ represents the noise schedule of the process, which is typically chosen as a monotonically increasing linear function of $t$~\cite{ho2020denoising}. The term $\boldsymbol{w}$ denotes the standard $d$-dimensional Wiener process. The data distribution is defined at $t=0$, i.e., $\boldsymbol{x}(0) \sim p_{\text{data}}$, while a simple and tractable distribution, such as an isotropic Gaussian, is achieved at $t=T$, i.e., $\boldsymbol{x}(T) \sim \mathcal{N} (\boldsymbol{0}, \bold{I})$.

The goal is to recover the data-generating distribution from the tractable distribution. This can be accomplished by formulating the corresponding reverse SDE for \cref{eq:forward-sde}, as derived in ~\cite{anderson1982reverse}:
\begin{align}
\label{eq:reverse-sde}
    d\boldsymbol{x} = \left[-\frac{\beta(t)}{2} \boldsymbol{x} - \beta(t)\nabla_{\boldsymbol{x}_t} \log p_t({\boldsymbol{x}_t})\right] dt + \sqrt{\beta(t)} d\bar{\boldsymbol{w}},
\end{align}
where $dt$ represents time running backward, and $d\bar{\boldsymbol{w}}$ corresponds to the standard Wiener process running in reverse. The drift function now depends on the time-dependent score function $\nabla_{\boldsymbol{x}_t} \log p_t(\boldsymbol{x}_t)$, which is approximated by a neural network $\boldsymbol{s}_\theta$ trained using denoising score matching~\cite{vincent2011connection}:
\begin{align}
\label{eq:dsm}
    \theta^* &= \arg \min_\theta \mathbb{E}_{t \sim U(\varepsilon, 1), \boldsymbol{x}(t) \sim p(\boldsymbol{x}(t)|\boldsymbol{x}(0)), \boldsymbol{x}(0) \sim p_{\text{data}}} \\
    &\left[\|\boldsymbol{s}_\theta(\boldsymbol{x}(t), t) - \nabla_{\boldsymbol{x}_t}\log p(\boldsymbol{x}(t)|\boldsymbol{x}(0))\|_2^2\right],
\end{align}
where $\varepsilon \simeq 0$ represents a small positive constant. Once the optimal parameters $\theta^*$ are obtained through \cref{eq:dsm}, the approximation $\nabla_{\boldsymbol{x}_t} \log p_t(\boldsymbol{x}_t) \simeq \boldsymbol{s}_{\theta^*}(\boldsymbol{x}_t, t)$ can be used as a plug-in estimate to replace the score function in \cref{eq:reverse-sde}.

Discretizing \cref{eq:reverse-sde} and solving it yields samples from the data distribution $p(\boldsymbol{x}_0)$, which is the ultimate goal of generative modeling. In discrete settings with $N$ time steps, we define $\boldsymbol{x}_i \triangleq \boldsymbol{x}(iT/N)$ and $\beta_i \triangleq \beta(iT/N)$. Following ~\cite{ho2020denoising}, we then introduce $\alpha_i \triangleq 1 - \beta_i$ and $\bar{\alpha}_i \triangleq \prod_{j=1}^i \alpha_j$.

\subsection{Diffusion Models for Solving Inverse Problems} \label{sec:setup}
We consider the linear inverse problems for reconstructing an unknown signal $\boldsymbol{x}_0 \in \mathbb{R}^d$ from noisy measurements $\boldsymbol{y}\in \mathbb{R}^m$:
\begin{equation}
\label{eq:forward-model}
    \boldsymbol{y} = \bold{A}\boldsymbol{x}_0 + \boldsymbol{n},
\end{equation}
where $\bold{A}\in \mathbb{R}^{m\times d}$ is a known measurement operator and $\boldsymbol{n}\sim \mathcal{N}(\boldsymbol{0}, \sigma^2 \bold{I})$ is an i.i.d. additive Gaussian noise with a known standard deviation of $\sigma$. This gives a likelihood function $p(\boldsymbol{y}|\boldsymbol{x}_0)=\mathcal{N}(\boldsymbol{y}|\bold{A} \boldsymbol{x}_0, \sigma^2 \bold{I})$.

Usually, we are interested in the case when $m < d$, which follows many real-world scenarios. When $m < d$, the problem is ill-posed and some kind of \textit{prior} is necessary to obtain a meaningful solution. In the Bayesian framework, one utilizes $p(\boldsymbol{x}_0)$ as the {\em prior}, and samples from the {\em posterior} $p(\boldsymbol{x}_0|\boldsymbol{y})$, where the relationship is formally established with the Bayes' rule: $p(\boldsymbol{x}_0|\boldsymbol{y}) = p(\boldsymbol{y}|\boldsymbol{x}_0)p(\boldsymbol{x}_0)/p(\boldsymbol{y})$. Leveraging the diffusion model as the prior, it is straightforward to modify \cref{eq:reverse-sde} to arrive at the reverse diffusion sampler for sampling from the posterior distribution:
\begin{align} \label{eq:reverse-sde-posterior}
d \boldsymbol{x} = \Big[ & -\frac{\beta(t)}{2} \boldsymbol{x}  - \beta(t)(\nabla_{\boldsymbol{x}_t} \log p_t(\boldsymbol{x}_t) \nonumber \\
& + \nabla_{\boldsymbol{x}_t} \log p_t(\boldsymbol{y}|\boldsymbol{x}_t))\Big] dt + \sqrt{\beta(t)} d\bar{\boldsymbol{w}},
\end{align}
where we have used the fact that
\begin{align}
\label{eq:grad_log_bayes}
    \nabla_{\boldsymbol{x}_t} \log p_t(\boldsymbol{x}_t|\boldsymbol{y}) = \nabla_{\boldsymbol{x}_t} \log p_t(\boldsymbol{x}_t) + \nabla_{\boldsymbol{x}_t} \log p_t(\boldsymbol{y}|\boldsymbol{x}_t).
\end{align}
In \cref{eq:reverse-sde-posterior}, there are two terms that need to be computed: the score function $\nabla_{\boldsymbol{x}_t} \log p_t(\boldsymbol{x}_t)$ and the likelihood $\nabla_{\boldsymbol{x}_t} \log p_t(\boldsymbol{y}|\boldsymbol{x}_t)$. To compute the former, involving $p_t(\boldsymbol{x}_t)$, we can directly use the pre-trained score function $\boldsymbol{s}_{\theta^*}$. However, the latter term is challenging to obtain in closed-form due to its dependence on time $t$ (note that there is only an explicit relationship between $\boldsymbol{y}$ and $\boldsymbol{x}_0$). As such, the likelihood $p_t(\boldsymbol{y}|\boldsymbol{x}_t)$ shall be estimated. One approach to achieve this estimation is to factorize $p(\boldsymbol{y}|\boldsymbol{x}_t)$ as follows:
\begin{align}
    p(\boldsymbol{y}|\boldsymbol{x}_t) &= \int p(\boldsymbol{y}|\boldsymbol{x}_0, \boldsymbol{x}_t)p(\boldsymbol{x}_0|\boldsymbol{x}_t) d\boldsymbol{x}_0 \notag \\
    &= \int p(\boldsymbol{y}|\boldsymbol{x}_0)p(\boldsymbol{x}_0|\boldsymbol{x}_t) d\boldsymbol{x}_0, 
\label{eq:factorize_yxt}
\end{align}
where the second equality comes from that $\boldsymbol{y}$ {and $\boldsymbol{x}_t$ are conditionally independent on} $\boldsymbol{x}_0$. Assuming that the measurement model $p(\boldsymbol{y}|\boldsymbol{x}_0)$ is known, based on \cref{eq:factorize_yxt}, one can approximate $p(\boldsymbol{y}|\boldsymbol{x}_t)$ by approximating $p(\boldsymbol{x}_0|\boldsymbol{x}_t)$. Although the exact form of $p(\boldsymbol{x}_0|\boldsymbol{x}_t)$ is intractable, the conditional mean of $ \boldsymbol{x}_0 $ given $\boldsymbol{x}_t $ under  $p(\boldsymbol{x}_0|\boldsymbol{x}_t)$, denoted by $\Tilde{\boldsymbol{x}}_0 = \mathbb{E} (\boldsymbol{x}_0 | \boldsymbol{x}_t)$, can be analytically obtained using Tweedie’s formula \citep{efron2011tweedie}:

\begin{proposition} [Tweedie's formula] \label{prop:Tweedie}
Let $p({\boldsymbol{y}}|{\boldsymbol{\eta}})$ belong to the exponential family distribution
\begin{align}
    p({\boldsymbol{y}}|{\boldsymbol{\eta}}) = p_0({\boldsymbol{y}}) \exp({\boldsymbol{\eta}}^\top T({\boldsymbol{y}}) - \varphi({\boldsymbol{\eta}})),
\end{align}
where ${\boldsymbol{\eta}}$ is the canonical vector of the family, $T({\boldsymbol{y}})$ is some function of $\boldsymbol{y}$, and $\varphi({\boldsymbol{\eta}})$ is the cumulant generation function which normalizes the density, and $p_0({\boldsymbol{y}})$ is the density up to the scale factor when ${\boldsymbol{\eta}} = \textbf{0}$.
Then, the posterior mean $\mathbb{E}[{\boldsymbol{\eta}}|\boldsymbol{y}]$ should satisfy
\begin{equation}\label{eq:teq}
 (\nabla_{\boldsymbol{y}} T({\boldsymbol{y}}))^\top \mathbb{E}[{\boldsymbol{\eta}}|\boldsymbol{y}] = \nabla_{\boldsymbol{y}} \log p({\boldsymbol{y}})-\nabla_{\boldsymbol{y}}\log p_0({\boldsymbol{y}}).
\end{equation}    
\end{proposition}

Now, note that in DDPM sampling process we have
\begin{align}
p(\boldsymbol{x}_t|\boldsymbol{x_0}) = \frac{1}{(2\pi(1-\bar\alpha(t)))^{d/2}}\exp\left(-\frac{\|\boldsymbol{x}_t-\sqrt{\bar\alpha(t)}\boldsymbol{x_0}\|^2}{2(1-\bar\alpha(t))}\right),
\end{align}
which is a Gaussian distribution.
The corresponding canonical decomposition is then given by
\begin{align}
p(\boldsymbol{x}_t|\boldsymbol{x_0}) = p_0(\boldsymbol{x}_t) \exp\left(\boldsymbol{x_0}^\top T(\boldsymbol{x}_t)- \varphi(\boldsymbol{x_0}) \right),
\end{align}
where
\begin{subequations} \label{eq:DDPMeq}
\begin{align}
p_0(\boldsymbol{x}_t)&:= \frac{1}{(2\pi(1-\bar\alpha(t)))^{d/2}}\exp\left(-\frac{\|\boldsymbol{x}_t\|^2}{2(1-\bar\alpha(t))}\right),\\
T(\boldsymbol{x}_t) &:=\frac{\sqrt{\bar\alpha(t)}}{1-\bar\alpha(t)}\boldsymbol{x}_t,\\
\varphi(\boldsymbol{x_0})&:=\frac{\bar\alpha(t)\|\boldsymbol{x_0}\|^2}{2(1-\bar\alpha(t))}.
\end{align}
\end{subequations}

Therefore, substituting \cref{eq:DDPMeq} in \cref{eq:teq}, we obtain
\begin{align*}
\frac{\sqrt{\bar\alpha(t)}}{1-\bar\alpha(t)} \Tilde{\boldsymbol{x}}_0 = \nabla_{\boldsymbol{x}_t}\log p_t(\boldsymbol{x}_t) +\frac{1}{1-\bar\alpha(t)}\boldsymbol{x}_t,
\end{align*}
which leads to 
\begin{align}
\label{eq:tr}
\Tilde{\boldsymbol{x}}_0 = \frac{1}{\sqrt{\bar\alpha(t)}}\left(\boldsymbol{x}_t+(1-\bar\alpha(t))\nabla_{\boldsymbol{x}_t}\log p_t(\boldsymbol{x}_t)\right). 
\end{align}

Then, two recent studies deploy the expected value $\Tilde{\boldsymbol{x}}_0$ to approximate $p(\boldsymbol{x}_0|\boldsymbol{x}_t)$, which we discuss their methodologies in detail in the following.

\noindent \textbf{(I) DPS~\cite{chungdiffusion}} Denoising posterior sampling (DPS) approximates $p_{t}(\boldsymbol{x}_0|\boldsymbol{x}_t)$ using a delta distribution $\delta(\boldsymbol{x}_0-\Tilde{\boldsymbol{x}}_0)$ centered at the posterior mean estimate $\Tilde{\boldsymbol{x}}_0$. As such, the likelihood $p_t(\boldsymbol{y}|\boldsymbol{x}_t)$ is approximated by
\begin{align}
p_t(\boldsymbol{y}|\boldsymbol{x}_t) &\approx \int p(\boldsymbol{y}|\boldsymbol{x}_0)\delta(\boldsymbol{x}_0-\Tilde{\boldsymbol{x}}_0) \mathrm{d}\boldsymbol{x}_0 \nonumber\\
&= p(\boldsymbol{y}|\boldsymbol{x}_0=\Tilde{\boldsymbol{x}}_0). \label{eq:dps}
\end{align}
However, directly using Eq.~\cref{eq:dps} does not perform well in practice, and \citet{chungdiffusion} empirically adjusts the strength of guidance by approximating the likelihood score $\nabla_{\boldsymbol{x}_t}\log p_t(\boldsymbol{y}|\boldsymbol{x}_t)$ with $-\zeta_t \nabla_{\boldsymbol{x}_t} \lVert \boldsymbol{y} - \bold{A} \Tilde{\boldsymbol{x}}_0\rVert_2^2$, where $\zeta_t = \zeta / \lVert \boldsymbol{y} - \bold{A} \Tilde{\boldsymbol{x}}_0\rVert_2$ with a hyper-parameter $\zeta$. 

\noindent \textbf{(II) $\Pi\text{GDM}$~\cite{song2023pseudoinverse}} The delta distribution used in DPS is a very rough approximation to $p_t(\boldsymbol{x}_0|\boldsymbol{x}_t)$ as it completely disregards the uncertainty of $\boldsymbol{x}_0$ given $\boldsymbol{x}_t$. 
As $t$ increases, the uncertainty in $p_t(\boldsymbol{x}_0|\boldsymbol{x}_t)$ becomes larger and is closed to the original data distribution $p(\boldsymbol{x}_0)$. Thus, it is more reasonable to choose a positive $r_t$.
In $\Pi\text{GDM}$, $r_t$ is heuristically selected as $\sqrt{\sigma_t^2/(1 + \sigma_t^2)}$ under the assumption that $p(\boldsymbol{x}_0)$ is the standard normal distribution $\mathcal{N}(\boldsymbol{0}, \bold{I})$. In such case, the likelihood $p_t(\boldsymbol{y}|\boldsymbol{x}_t)$ is approximated by
\begin{align}
p_t(\boldsymbol{y} |\boldsymbol{x}_t) &\approx\int\mathcal{N}(\boldsymbol{y}|\bold{A} \boldsymbol{x}_0, \sigma^2 \bold{I}) \mathcal{N}(\boldsymbol{x}_0|\Tilde{\boldsymbol{x}}_0, r_t^2 \bold{I}) \mathrm{d}\boldsymbol{x}_0\nonumber\\
&= \mathcal{N}(\boldsymbol{y}|\bold{A}\Tilde{\boldsymbol{x}}_0, \sigma^2 \bold{I} + r_t^2\bold{A} \bold{A}^\top). \label{eq:pgdm-likelihood}
\end{align}

\section{Covariance-Aware Diffusion
Posterior Sampling} \label{sec:meth}
In this section, we aim to improve the approximation of the reverse process $ p(\boldsymbol{x}_0 | \boldsymbol{x}_t) $ compared to DPS and $ \Pi $GDM. Specifically, instead of heuristically approximating the conditional covariance of $ \boldsymbol{x}_0 $ given $ \boldsymbol{x}_t $ as done by $ \Pi $GDM, we derive a closed-form formula for it. To this end, we first introduce the following theorem.

\begin{theorem} \label{theorem:var}
Under the same conditions as in \cref{prop:Tweedie}, the posterior covariance  $\text{Cov} ({\boldsymbol{\eta}}|\boldsymbol{y})$  satisfies
\begin{align}
(\nabla_{\boldsymbol{y}} & T({\boldsymbol{y}}))^\top   \text{Cov}({\boldsymbol{\eta}}|\boldsymbol{y}) \nabla_{\boldsymbol{y}} T({\boldsymbol{y}}) \nonumber \\
& =  \nabla^2_{\boldsymbol{y}} \log p({\boldsymbol{y}}) - \nabla^2_{\boldsymbol{y}} \log p_0({\boldsymbol{y}}) - \nabla^2_{\boldsymbol{y}} T({\boldsymbol{y}}) \odot \mathbb{E}({\boldsymbol{\eta}}|\boldsymbol{y}), \label{eq:var}  
\end{align}
where $\mathbb{E}({\boldsymbol{\eta}}|\boldsymbol{y})$ is obtained using Tweedie’s formula in \cref{prop:Tweedie}. Additionally, the operator $\odot$ denotes a contraction operation between the three dimensional tensor $\nabla^2_{\boldsymbol{y}} T({\boldsymbol{y}})$ and the vector $\mathbb{E}({\boldsymbol{\eta}}|\boldsymbol{y})$. Specifically, assuming that $\boldsymbol{y} \in \mathbb{R}^r$ and $\boldsymbol{\eta} \in \mathbb{R}^k$ (which yields \(\nabla^2_{\boldsymbol{y}} T({\boldsymbol{y}}) \in \mathbb{R}^{r \times r \times k}\) and \(\mathbb{E}[{\boldsymbol{\eta}}|\boldsymbol{y}] \in \mathbb{R}^k\)), then $\nabla^2_{\boldsymbol{y}} T({\boldsymbol{y}}) \odot \mathbb{E}({\boldsymbol{\eta}}|\boldsymbol{y}) \in \mathbb{R}^{r\times r}$ is defined as
\begin{align}
[\nabla^2_{\boldsymbol{y}} T({\boldsymbol{y}}) \odot \mathbb{E}({\boldsymbol{\eta}}|\boldsymbol{y})]_{ij} = \sum_{k} [\nabla^2_{\boldsymbol{y}} T({\boldsymbol{y}})]_{ijk} \mathbb{E}({\eta_k}|\boldsymbol{y}).    
\end{align}  
\end{theorem}
\begin{proof}
Please refer to the \textit{\textit{Supplementary}} materials.    
\end{proof}
Next, we use \cref{theorem:var} for DDPMs to find a closed-form expression for the conditional covariance $\text{Cov} (\boldsymbol{x_0} | \boldsymbol{x}_t)$.

\begin{corollary} \label{cor:DDPM}
Using \cref{eq:DDPMeq} in \cref{theorem:var} we obtain
\begin{align}
\big( \frac{\sqrt{\bar\alpha(t)}}{1-\bar\alpha(t)} \big)^2 \text{Cov} (\boldsymbol{x_0} | \boldsymbol{x}_t) &= \nabla^2_{\boldsymbol{x}_t}\log p_t(\boldsymbol{x}_t) + \frac{1}{1-\bar\alpha(t)} \bold{I},
\end{align}
which leads to
\begin{align} \label{eq:varlast}
\text{Cov} (\boldsymbol{x_0} | \boldsymbol{x}_t) = \frac{1-\bar\alpha(t)}{\bar\alpha(t)} \Big( \bold{I} + (1-\bar\alpha(t)) \nabla^2_{\boldsymbol{x}_t}\log p_t(\boldsymbol{x}_t) \Big). 
\end{align}    
\end{corollary}

As seen in \cref{eq:varlast}, $\text{Cov} (\boldsymbol{x_0} | \boldsymbol{x}_t)$ depends on the  Hessian $\bold{H}_t = \nabla^2_{\boldsymbol{x}_t}\log p_t(\boldsymbol{y}|\boldsymbol{x}_t)$. Nevertheless, 
the Hessian $\bold{H}_t$ is not available for DDPMs (note that DDPMs only return the score function $\nabla_{\boldsymbol{x}_t}\log p_t(\boldsymbol{y}|\boldsymbol{x}_t)$). To this aim, we shall resort to approximating $\bold{H}_t$. Given that the sampling process in DDPMs is inherently time-consuming, our approximation method for $\bold{H}_t$ needs to be straightforward to avoid adding complexity to the sampling process. Thus, we approximate the Hessian $ \bold{H}_t $ with a diagonal matrix $ \tilde{\bold{H}}_t $, where the diagonal elements are obtained from the gradient vector $ \nabla_{\boldsymbol{x}_t} \log p_t(\boldsymbol{y} |\boldsymbol{x}_t) $ using the finite difference method:  
\begin{align}
\Tilde{\bold{H}}_t =
\frac{\nabla_{\boldsymbol{x}_t}\log p_t(\boldsymbol{x}_t) - \nabla_{\boldsymbol{x}_{t-1}}\log p_{t-1}(\boldsymbol{x}_{t-1})}{t - (t-1)}.
\end{align}
In this respect, $\Tilde{\bold{H}}_t$ can be readily obtained from the existing pre-trained DDPMs, and thereby does not incur any extra complexities. Then, using $\Tilde{\bold{H}}_t$ in \cref{eq:varlast}, we obtain an approximation to the true $\text{Cov}(\boldsymbol{x}_0 | \boldsymbol{x}_t)$ as 
\begin{align} \label{eq:covapprox}
\Tilde{\boldsymbol{\Sigma}}_t = \frac{1-\bar\alpha(t)}{\bar\alpha(t)} \Big( \bold{I} + (1-\bar\alpha(t)) \Tilde{\bold{H}}_t \Big).    
\end{align}

Henceforth, we approximate $p_{t}(\boldsymbol{x}_0|\boldsymbol{x}_t)$ as follows
\begin{align}
p_{t}(\boldsymbol{x}_0|\boldsymbol{x}_t) \approx \mathcal{N} (\Tilde{\boldsymbol{x}}_0, \Tilde{\boldsymbol{\Sigma}}_t),    
\end{align}
which leads to the following approximation for the likelihood $p_t(\boldsymbol{y} |\boldsymbol{x}_t)$
\begin{align}
p_t(\boldsymbol{y} |\boldsymbol{x}_t) &\approx\int\mathcal{N}(\bold{A} \boldsymbol{x}_0, \sigma^2 \bold{I}) \mathcal{N}(\Tilde{\boldsymbol{x}}_0, \Tilde{\boldsymbol{\Sigma}}_t) \mathrm{d}\boldsymbol{x}_0\nonumber\\
&= \mathcal{N}(\bold{A}\Tilde{\boldsymbol{x}}_0, \sigma^2 \bold{I} +  \bold{A} \Tilde{\boldsymbol{\Sigma}}_t \bold{A}^\top). \label{eq:our-likelihood}
\end{align}

Lastly, the gradient of the log-likelihood, $ \nabla_{\boldsymbol{x}_t} \log p_t(\boldsymbol{y} |\boldsymbol{x}_t) $, can be approximated using the Jacobian-vector product, similar to the approach in \cite{song2023pseudoinverse}:
\begin{align}
& \nabla_{\boldsymbol{x}_t} \log p_t(\boldsymbol{y}|\boldsymbol{x}_t)  \nonumber \\
& \approx  \big( \nabla_{\boldsymbol{x}_t} \Tilde{\boldsymbol{x}}_0   \big) \bold{A}^\top \big( \sigma^2 \bold{I} + \bold{A} \Tilde{\boldsymbol{\Sigma}}_t \bold{A}^\top \big)^{-1} \big( \boldsymbol{y} - \bold{A} \Tilde{\boldsymbol{x}}_0 \big). \label{eq:loglike} 
\end{align}
To find the term $\big( \nabla_{\boldsymbol{x}_t} \Tilde{\boldsymbol{x}}_0 \big)$ in \cref{eq:loglike}, we take a derivative w.r.t. $\boldsymbol{x}_t$ from both sides of \cref{eq:tr} to obtain
\begin{align}
\nabla_{\boldsymbol{x}_t} & \Tilde{\boldsymbol{x}}_0 =   \frac{1}{\sqrt{\bar\alpha(t)}}\left(\bold{I}+(1-\bar\alpha(t))\nabla^2_{\boldsymbol{x}_t}\log p_t(\boldsymbol{x}_t)\right)  \\
& \approx   \frac{1}{\sqrt{\bar\alpha(t)}}\left(\bold{I}+(1-\bar\alpha(t)) \Tilde{\bold{H}}_t \right) \\
& = \frac{\sqrt{\bar\alpha(t)}}{1-\bar\alpha(t)} \Tilde{\boldsymbol{\Sigma}}_t. \label{eq:derx0}
\end{align}
Thus, using \cref{eq:derx0} in \cref{eq:loglike} we obtain
\begin{align}
& \nabla_{\boldsymbol{x}_t} \log p_t(\boldsymbol{y}|\boldsymbol{x}_t)  \nonumber \\
& \approx  \frac{\sqrt{\bar\alpha(t)}}{1-\bar\alpha(t)} \Tilde{\boldsymbol{\Sigma}}_t  \bold{A}^\top \big( \sigma^2 \bold{I} + \bold{A} \Tilde{\boldsymbol{\Sigma}}_t \bold{A}^\top \big)^{-1} \big( \boldsymbol{y} - \bold{A} \Tilde{\boldsymbol{x}}_0 \big). \label{eq:loglike2} 
\end{align}

Note that in general, computing the inverse $\big( \sigma^2 \bold{I} + \bold{A} \Tilde{\boldsymbol{\Sigma}}_t \bold{A}^\top \big)^{-1}$ for high-dimensional images is computationally challenging. To this end, we use a low complexity method to find this inverse. In particular, we first define $\boldsymbol{\lambda} = \big( \sigma^2 \bold{I} + \bold{A} \Tilde{\boldsymbol{\Sigma}}_t \bold{A}^\top \big)^{-1} \big( \boldsymbol{y} - \bold{A} \Tilde{\boldsymbol{x}}_0 \big)$. As the matrix  $\sigma^2 \bold{I} + \bold{A} \Tilde{\boldsymbol{\Sigma}}_t \bold{A}^\top$ is is symmetric and positive-definite, then $\boldsymbol{\lambda}$ can be obtained by solving the following linear equation
\begin{align}
\big( \sigma^2 \bold{I} + \bold{A} \Tilde{\boldsymbol{\Sigma}}_t \bold{A}^\top \big) \boldsymbol{\lambda} = \boldsymbol{y} - \bold{A} \Tilde{\boldsymbol{x}}_0.     
\end{align}
Therefore, $\boldsymbol{\lambda}$ can be computed with acceptable precision using a sufficient number of conjugate gradient (CG) iterates\footnote{In the experiments, we utilize the black-box CG method implemented in `scipy.sparse.linalg.cg`, with a tolerance of $ \text{tol}=1\text{e}{-4} $.}.

Once $ \nabla_{\boldsymbol{x}_t} \log p_t(\boldsymbol{y} \mid \boldsymbol{x}_t) $ is computed, the gradient of the posterior in \cref{eq:grad_log_bayes} can be calculated and used during the reverse denoising pass. This will slightly modify the unconditional DDPMs reverse pass as shown in \cref{alg:CA-DPS}. Note that the only difference of CA-DPS algorithm with an unconditional sampling is in line 8 (written in blue), where the conditioning is applied.

\begin{algorithm}[H]
   \caption{CA-DPS}
   \label{alg:CA-DPS}
\begin{algorithmic}[1]
   \STATE   {\bfseries Input:} The number of iterations $N$, $\boldsymbol{y}$, noise levels $\{ \Tilde{\sigma} \}$. 
   \STATE $\boldsymbol{x}_N \sim \mathcal{N} (\bold{0}, \bold{I})$
   \FOR{$t=N-1, N-2,\dots,0$}
   \STATE  $\hat{\boldsymbol{s}} \leftarrow \boldsymbol{s}_{\theta} (\boldsymbol{x}_t,t)$ 
   \STATE $ \Tilde{\boldsymbol{x}}_0 \leftarrow \frac{1}{\sqrt{\bar\alpha_t}}\left(\boldsymbol{x}_t+(1-\bar\alpha_t)\hat{\boldsymbol{s}} \right)$
   \STATE $\boldsymbol{z} \sim \mathcal{N} (\bold{0}, \bold{I})$.
   \STATE  $\boldsymbol{x}_{t-1}^\prime \leftarrow \frac{\sqrt{\alpha_t}(1-\bar{\alpha}_{t-1})}{1 - \bar{\alpha}_t}\boldsymbol{x}_t + \frac{\sqrt{\bar{\alpha}_{t-1}}\beta_t}{1 - \bar{\alpha}_t}\Tilde{\boldsymbol{x}}_0 +  {\tilde{\sigma}_t \boldsymbol{z}}$.
   \STATE  \color{blue}  
   \small $\boldsymbol{x}_{t-1} \leftarrow \boldsymbol{x}_{t-1}^\prime - \frac{\sqrt{\bar\alpha_t}}{1-\bar\alpha_t} \Tilde{\boldsymbol{\Sigma}}_t \bold{A}^\top \big( \sigma^2 \bold{I} + \bold{A} \Tilde{\boldsymbol{\Sigma}}_t \bold{A}^\top \big)^{-1} \big( \boldsymbol{y} - \bold{A} \Tilde{\boldsymbol{x}}_0 \big)$. \normalsize \color{black}
   \ENDFOR
   \STATE  {\bfseries Output:} $\boldsymbol{x}_0$
\end{algorithmic}
\end{algorithm}

\begin{table*}[!t]
\renewcommand{\baselinestretch}{1.0}
\renewcommand{\arraystretch}{1.0}
\setlength{\tabcolsep}{2.2pt}
\centering
\resizebox{1\textwidth}{!}{\begin{tabular}{@{}c|cccccccccccccccc@{}}
\toprule
\multirow{2}{*}{\textbf{Dataset}} & \multirow{2}{*}{\textbf{Method}}      & \multicolumn{3}{c}{\textbf{Inpaint~(Random)}} & \multicolumn{3}{c}{\textbf{Inpaint~(Box)}} & \multicolumn{3}{c}{\textbf{Deblur~(Gaussian)}} & \multicolumn{3}{c}{\textbf{Deblur~(Motion)}} & \multicolumn{3}{c}{\textbf{SR~($4\times$)}} \\
\cmidrule(lr){3-5}  \cmidrule(lr){6-8}  \cmidrule(lr){9-11}  \cmidrule(lr){12-14} \cmidrule(lr){15-17}
&      & FID~$\downarrow$ & LPIPS~$\downarrow$ & SSIM~$\uparrow$ & FID~$\downarrow$ & LPIPS~$\downarrow$ & SSIM~$\uparrow$ & FID~$\downarrow$ & LPIPS~$\downarrow$ & SSIM~$\uparrow$ & FID~$\downarrow$ & LPIPS~$\downarrow$ & SSIM~$\uparrow$ & FID~$\downarrow$ & LPIPS~$\downarrow$ & SSIM~$\uparrow$  \\
\midrule
\multirow{7}{*}{FFHQ}  & DPS  & \underline{21.19} & \underline{0.212} & \underline{0.851} & \underline{33.12} & \underline{0.168}& \textbf{0.873} & 44.05 & 0.257 &0.811 & 39.92 & 0.242 & 0.859 & 39.35 & 0.214 & 0.852 \\ & $\Pi$GDM & 21.27 & 0.221 & 0.840 & 34.79 & 0.179& 0.860 & \underline{40.21} & \underline{0.242}&0.825 & \underline{33.24} & \underline{0.221} & \underline{0.887} & \underline{34.98} & \underline{0.202} & 0.854  \\
& DDRM & 69.71 & 0.587 & 0.319 & 42.93 &0.204& \underline{0.869}& 74.92 &0.332& 0.767 & $-$ & $-$ & $-$ &62.15 &0.294 &0.835\\
& MCG & 29.26 & 0.286 &  0.751 & 40.11 &0.309& 0.703&  101.2 &0.340&0.051& $-$ & $-$ & $-$ &87.64 &0.520  &0.559\\ 
& PnP-ADMM  & 123.6 & 0.692 &  0.325 & 151.9 &0.406& 0.642 &90.42& 0.441&\underline{0.812}& $-$ & $-$ & $-$ &66.52& 0.353 &\textbf{0.865}\\
& Score-SDE & 76.54 & 0.612 & 0.437 & 60.06 &0.331& 0.678& 109.0& 0.403& 0.109& $-$ & $-$ & $-$ &96.72 &0.563 & 0.617\\
& ADMM-TV & 181.5 & 0.463 & 0.784 & 68.94 &0.322& 0.814&  186.7& 0.507&  0.801& $-$ & $-$ & $-$ &110.6& 0.428 & 0.803\\
& \textbf{CA-DPS} & \textbf{20.14} & \textbf{0.207} &  \textbf{0.881} & \textbf{26.33} &\textbf{0.132}& \underline{0862} & \textbf{32.74} & \textbf{0.238} & \textbf{0.832} & \textbf{27.59} & \textbf{0.217} & \textbf{0.921} & \textbf{28.41} & \textbf{0.196} & \underline{0.855} \\
\midrule
\midrule
\multirow{7}{*}{ImageNet} & DPS & \underline{35.87} &\underline{0.303}& \underline{0.739} & \underline{38.82} &0.262& 0.794 & 62.72 &0.444& 0.706 &56.08 &0.389 &0.634& \underline{50.66} & \underline{0.337} & \underline{0.781}\\ 
& $\Pi$GDM & 41.82 &0.356& 0.705 & 42.26 &0.284& 0.752 & \underline{59.79} &\underline{0.425}& \textbf{0.717} &\underline{54.18} &\underline{0.373} &\textbf{0.675}& 54.26 & 0.352 & 0.765\\
& DDRM & 114.9& 0.665&  0.403& 45.95& \underline{0.245}& \textbf{0.814} &63.02 &0.427&0.705& $-$ & $-$ & $-$ &59.57 &0.339 & \textbf{0.790}\\
& MCG & 39.19& 0.414&  0.546&  39.74 &0.330& 0.633  &95.04& 0.550& 0.441& $-$ & $-$ & $-$ &144.5 & 0.637 &0.227\\ 
& PnP-ADMM& 114.7 &0.677& 0.300&  78.24 &0.367&  0.657 &100.6 &0.519&0.669& $-$ & $-$ & $-$ &97.27 &0.433  & 0.761\\
& Score-SDE&  127.1 &0.659& 0.517& 54.07& 0.354&  0.612&  120.3& 0.667&0.436& $-$ & $-$ & $-$ &170.7 & 0.701 &0.256\\
& ADMM-TV& 189.3& 0.510&  0.676&  87.69 &0.319&  0.785 & 155.7 &0.588& 0.634& $-$ & $-$ & $-$ &130.9 &0.523 & 0.679\\
& \textbf{CA-DPS} & \textbf{32.37} & \textbf{0.214} & \textbf{0.755} & \textbf{33.24} & \textbf{0.247} & \underline{0.807} & \textbf{56.36} & \textbf{0.391} & \underline{0.712} & \textbf{52.06} & \textbf{0.352} & \underline{0.644} & \textbf{47.30} & \textbf{0.316} &  0.777\\
\bottomrule
\end{tabular}}
\caption{\textbf{Quantitative results on the $1k$ validation set of FFHQ $256\times256$ and ImageNet $256\times256$ dataset.} We use \textbf{bold} and \underline{underline} for the best and second best results, respectively.}\label{tab:quant}
\end{table*}

\begin{figure*}[!t]
\centering  \includegraphics[width=0.9\textwidth]{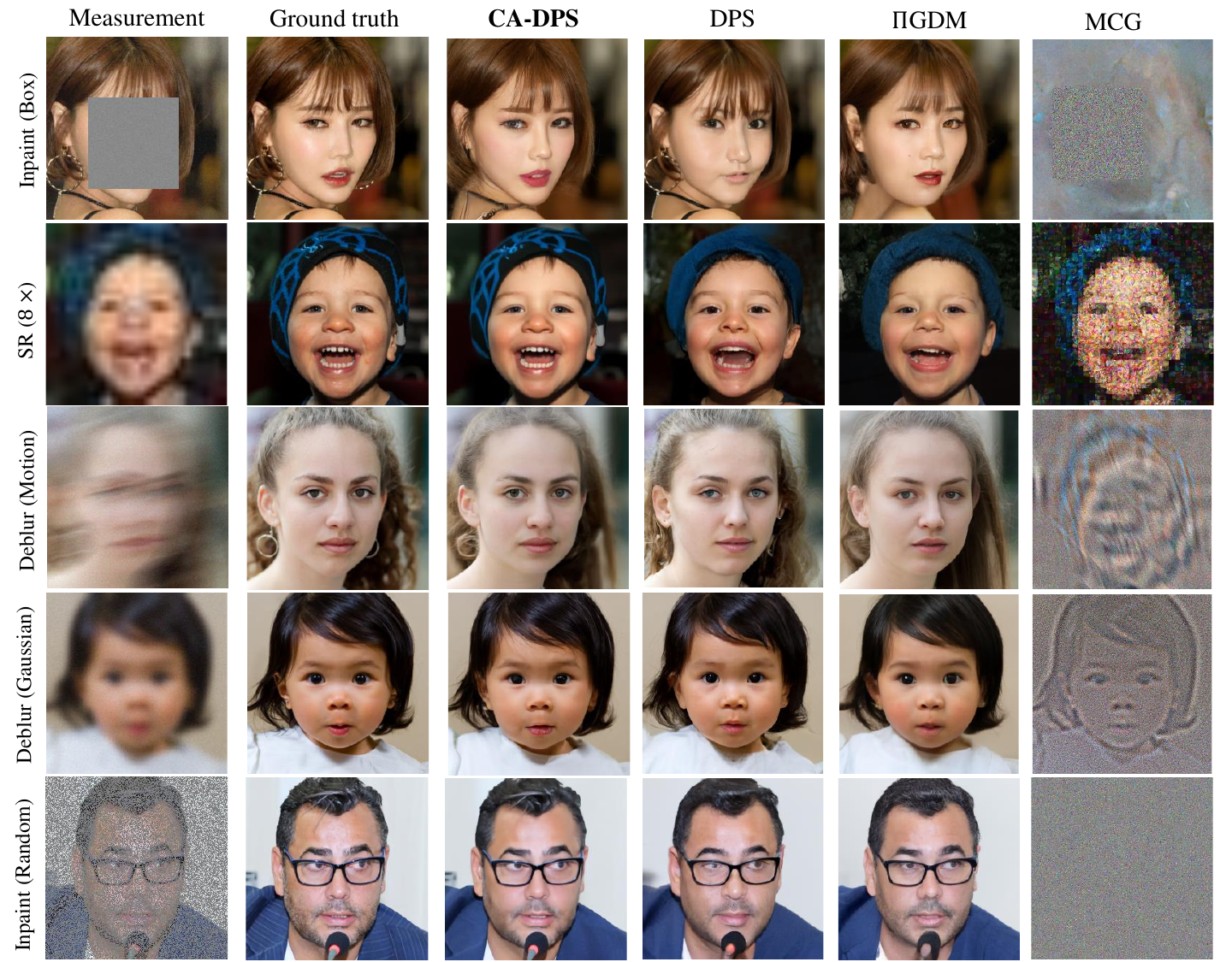}
\vspace{-2mm}
  \caption{Qualitative results on FFHQ dataset.
} \label{fig:fig1}
  \vspace{-2mm}
\end{figure*}

\section{Experiments} \label{sec:exp}
In this section, to demonstrate the superior performance of CA-DPS compared to state-of-the-art alternatives, we evaluate its effectiveness across a range of inverse problems---including inpainting, deblurring, and super-resolution---using two popular datasets. Specifically, we present quantitative and qualitative results in \cref{sec:Quantitative} and \cref{sec:Qualitative}, respectively. To further illustrate that the superiority of CA-DPS stems from its improved approximation of the true posterior, we compare its ability to estimate the true posterior against benchmark methods in \cref{sec:toy}, using a toy dataset with a known posterior.

\subsection{Quantitative Results.}\label{sec:Quantitative}
\noindent $\bullet$ \textbf{Experimental setup.} 
Following \cite{chungdiffusion,dou2024diffusion}, we perform experiments on FFHQ 256$\times$256 \cite{karras2019style} and ImageNet 256$\times$256 datasets \cite{deng2009imagenet}, on 1k validation images each. All images are normalized to the range $[0, 1]$. For a fair comparison, we use the experimental settings in \cite{chungdiffusion} for all the methods. All measurements are corrupted by Gaussian noise with mean zero and $\sigma=0.05$. For the backward process during the inference, we set the number of time steps as $N = 1000$ and use the pre-trained score model from \cite{chungdiffusion} for the FFHQ dataset, and the score model from \cite{dhariwal2021diffusion} for the ImageNet dataset.

The measurement models used are mostly based on \cite{chungdiffusion}: (i) for box-type inpainting, we mask out a \(128 \times 128\) box region, and for random-type inpainting, we mask out 92\% of the total pixels (across all RGB channels); (ii) for super-resolution (SR), we perform bicubic downsampling; (iii) for Gaussian blur, we use a kernel size of \(61 \times 61\) with a standard deviation of 3.0, and for motion blur, we use randomly generated kernels from the code\footnote{\url{https://github.com/LeviBorodenko/motionblur}}, with a size of \(61 \times 61\) and an intensity value of 0.5 (these kernels are then convolved with the ground truth image to produce the measurements).

\noindent $\bullet$ \textbf{Benchmark methods.} We compare the performance of CA-DPS with the following benchmark methods: 
DPS \cite{chungdiffusion}, $\Pi$GDM \citep{song2023pseudoinverse}, denoising diffusion restoration models (DDRM)~\citep{kawar2022denoising}, manifold constrained gradients (MCG)~\citep{chung2022improving}, Plug-and-play alternating direction method of multipliers (PnP-ADMM)~\citep{chan2016plug} , Score-SDE~\citep{song2020score} and total-variation (TV) sparsity regularized optimization method (ADMM-TV). For a fair comparison, we used the same score function for all the different methods that are based on diffusion (i.e. CA-DPS, DPS, DDRM, MCG, score-SDE).

\noindent $\bullet$ \textbf{Evaluation metrics.} To evaluate different methods, we follow \cite{chungdiffusion} to use three metrics: (i) learned perceptual image patch similarity (LPIPS) \cite{zhang2018unreasonable}, (ii) Frechet inception distance (FID) \cite{heusel2017gans}, and (iii) structure similarity index measure (SSIM). These metrics enable a comprehensive assessment of image quality. All our experiments are carried out on a single A100 GPU.

\noindent $\bullet$ \textbf{Experimental results.}
The results for both datasets are listed in \cref{tab:quant}. The results demonstrate that CA-DPS outperforms baselines significantly in almost all the tasks. It is
remarkable that in the challenging inpainting tasks (box and random), CA-DPS achieves the best performance. When assessing performance across the three metrics, CA-DPS emerges as the front-runner in three of them, with superior results compared to the other benchmark methods.

\subsection{Qualitative Results} \label{sec:Qualitative}
In this section, we aim to visualize the reconstructed images from CA-DPS and compare them with those reconstructed from the benchmark methods. To this end, we randomly select five images from the FFHQ test dataset and corrupt them using the measurement methods discussed in \cref{sec:Quantitative}, with the exception that we apply $8 \times$ super-resolution for better visualization. The images are depicted in \cref{fig:fig1}, where each row corresponds to a different measurement method, and each column corresponds to a different benchmark method.

It is observed that the images reconstructed by CA-DPS show greater similarity to the ground truth images compared to those reconstructed by the benchmark methods. Specifically, for the super-resolution task, CA-DPS is the only method that successfully reconstructs the pattern of the hat. Additionally, in the motion-deblurring task, the eyes and other facial textures reconstructed by CA-DPS are much closer to the ground truth image.

\subsection{Toy Dataset} \label{sec:toy}
In this subsection, we aim to illustrate that the superior results obtained by CA-DPS arise from its enhanced approximation of the true posterior. To support this, we shall compare its ability to estimate the true posterior against that of benchmark methods.

To this end, we generate a toy dataset whose distribution $p_0(\boldsymbol{x}_0)$ is a mixture of 25 Gaussian distributions\footnote{We followed \cite{cardoso2023monte,boys2023tweedie} to generate this dataset.}. The means and variances for each mixture component are detailed in the \textit{Supplementary} materials, where we also explain how, for a given set of observations \( \boldsymbol{y} \), measurement matrix \( \bold{A} \), and noise standard deviation \( \sigma \), the target posterior can be computed exactly.

To assess the effectiveness of posterior sampling methods, we generate multiple measurement models $ (\boldsymbol{y}, \bold{A}) \in \mathbb{R}^{m} \times \mathbb{R}^{m \times d} $ for combinations of dimensions and observation noise levels $ (d, m, \sigma) \in \{8, 80, 800\} \times \{1, 2, 4\} \times \{10^{-2}, 10^{-1}, 10^{0}\} $, while each Gaussian mixture component is equally weighted. By choosing different dimension sizes, we aim to understand how posterior sampling methods perform across varying dimensions, while controlling the noise level allows us to evaluate how these methods perform at different signal-to-noise ratios. 

Next, we generate 1000 samples for each of the above scenarios ($3\times 3 \times 3 =27$ scenarios). Then, we use CA-DPS, $\Pi$GDM and DPS to estimate the posterior probability through 1000 denoising steps. Afterward, to evaluate how well each algorithm estimates the posterior distribution compared to the target posterior, we utilize the sliced Wasserstein (SW) distance \cite{kolouri2019generalized}. We calculate the SW distance using $ 10^{4} $ slices for 1000 samples.

\cref{tab:cov} shows the 95\% confidence intervals, derived from 20 randomly selected measurement models ($ \bold{A} $) for each parameter setting ($ d, m, \sigma $). In addition, \cref{fig:fig2} illustrates the first two dimensions of the estimated posterior distributions for the configuration (80, 1) from \cref{tab:cov}, using one randomly generated measurement model ($ \bold{A}, \sigma = 0.1 $). This visualization gives insights into how well the algorithms estimate the posterior distribution, showing that CA-DPS provides a more accurate estimate of the target posterior compared to $\Pi$GDM and DPS, as it captures all modes, whereas $\Pi$GDM and DPS do not.

\begin{table}[!t]
\renewcommand{\baselinestretch}{1.0}
\renewcommand{\arraystretch}{1.0}
\setlength{\tabcolsep}{2.2pt}
\centering
\resizebox{1\columnwidth}{!}{\begin{tabular}{l|l|ccccccccc}
\toprule
& $d$    & 8            & 8            & 8            & 80           & 80           & 80           & 800          & 800          & 800          \\
& $m$    & 1            & 2            & 4            & 1            & 2            & 4            & 1            & 2            & 4            \\ \midrule
\multirow{3}{*}{$\sigma=10^{-2}$} & CA-DPS & \textbf{2.2} & \textbf{1.5} &\textbf{0.5}& \textbf{2.9}& \textbf{1.7} &\textbf{0.4} &\textbf{3.3} &\textbf{2.5} &\textbf{0.3}\\
& DPS      & 4.7          & 1.8          & 0.7          & 5.6          & 3.2          & 1.2          & 5.8          & 3.5          & 1.4          \\
& $\Pi$GDM & 2.6          & 2.1          & 3.8          & 3.2          & 2.8          & 0.6          & 3.5          & 3.1          & 0.4 \\ \midrule
\multirow{3}{*}{$\sigma=10^{-1}$} & CA-DPS & \textbf{1.8}& \textbf{0.9} &\textbf{0.6}& \textbf{2.5}& \textbf{1.7} &\textbf{0.4}& \textbf{2.8}& \textbf{2.3}& \textbf{0.4} \\
&DPS      & 4.7          & 1.5          & 0.8          & 5.1          & 3.1          & 1.0          & 5.7          & 3.1          & 1.3          \\
&$\Pi$GDM & 2.2          & 1.6          & 3.8          & 2.9          & 2.7          & 0.6          & 3.3          & 2.7          & \textbf{0.4}          \\ \midrule
\multirow{3}{*}{$\sigma=10^{0}$} & CA-DPS & \textbf{1.2}& \textbf{1.9}& \textbf{0.9} &1.7 &\textbf{1.2}& \textbf{0.8}& \textbf{1.6} &\textbf{1.5} &0.7  \\
&DPS      & 5.2          & 3.5          & 2.5          & 6.9          & 3.9          & 1.7          & 6.8          & 4.7          & 0.9          \\
&$\Pi$GDM & 1.5          & 2.3          & 1.8          & \textbf{1.6} & 1.4          & 0.9 & 2.0          & 2.0          & \textbf{0.6} \\ 
\bottomrule
\end{tabular}}
\caption{SW distance between the true and estimated posterior on toy dataset.}\label{tab:cov}
\end{table}

\begin{figure}[!t]
\centering  \includegraphics[width=0.8\columnwidth]{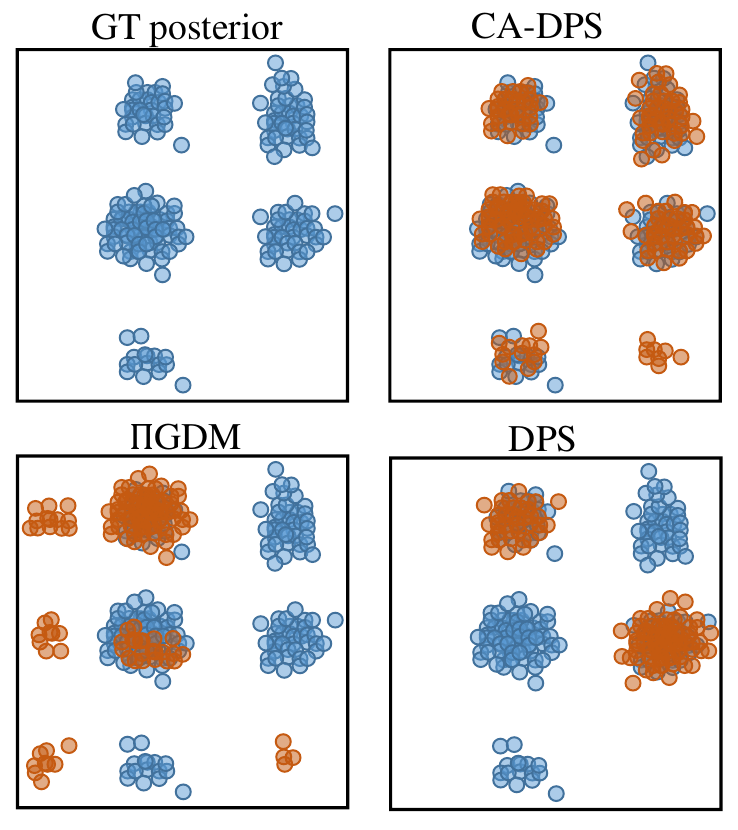}
\vspace{-2mm}
  \caption{Visualizing the first two dimensions of the estimated posterior distributions for the configuration ($d=80$, $m=1$, $\sigma = 10^{-1}$) for a randomly generated $\bold{A}$.
} \label{fig:fig2}
  \vspace{-2mm}
\end{figure}

\section{Conclusion}
In this paper, we proposed CA-DPS, a method designed to enhance the performance of DDPMs in solving inverse problems. To achieve this, we derived a closed-form expression for the covariance of reverse process in DDPMs. We then proposed a method based on finite differences to approximate this covariance, making it easily obtainable from existing pre-trained DDPMs. Utilizing the mean and the approximated covariance of the reverse process, we present a new approximation for the likelihood. Finally, we conducted three sets of experiments to demonstrate the superiority of CA-DPS: (i) quantitative evaluations using various metrics to assess the quality of reconstructed images, (ii) qualitative assessments by visualizing some of the reconstructed images, and (iii) testing the proximity of estimated posterior to the true posterior using a toy dataset with a known posterior.
\bibliography{aaai25}

\begin{thebibliography}{39}
\providecommand{\natexlab}[1]{#1}

\bibitem[{Anderson(1982)}]{anderson1982reverse}
Anderson, B.~D. 1982.
\newblock Reverse-time diffusion equation models.
\newblock \emph{Stochastic Processes and their Applications}, 12(3): 313--326.

\bibitem[{Bora et~al.(2017)Bora, Jalal, Price, and Dimakis}]{bora2017compressed}
Bora, A.; Jalal, A.; Price, E.; and Dimakis, A.~G. 2017.
\newblock Compressed sensing using generative models.
\newblock In \emph{International conference on machine learning}, 537--546. PMLR.

\bibitem[{Boys et~al.(2023)Boys, Girolami, Pidstrigach, Reich, Mosca, and Akyildiz}]{boys2023tweedie}
Boys, B.; Girolami, M.; Pidstrigach, J.; Reich, S.; Mosca, A.; and Akyildiz, O.~D. 2023.
\newblock Tweedie moment projected diffusions for inverse problems.
\newblock \emph{arXiv preprint arXiv:2310.06721}.

\bibitem[{Cardoso et~al.(2023)Cardoso, Idrissi, Corff, and Moulines}]{cardoso2023monte}
Cardoso, G.; Idrissi, Y. J.~E.; Corff, S.~L.; and Moulines, E. 2023.
\newblock Monte Carlo guided diffusion for Bayesian linear inverse problems.
\newblock \emph{arXiv preprint arXiv:2308.07983}.

\bibitem[{Chan, Wang, and Elgendy(2016)}]{chan2016plug}
Chan, S.~H.; Wang, X.; and Elgendy, O.~A. 2016.
\newblock Plug-and-play ADMM for image restoration: Fixed-point convergence and applications.
\newblock \emph{IEEE Transactions on Computational Imaging}, 3(1): 84--98.

\bibitem[{Choi et~al.(2021)Choi, Kim, Jeong, Gwon, and Yoon}]{choi2021ilvr}
Choi, J.; Kim, S.; Jeong, Y.; Gwon, Y.; and Yoon, S. 2021.
\newblock {ILVR}: Conditioning Method for Denoising Diffusion Probabilistic Models.
\newblock In \emph{Proceedings of the IEEE/CVF International Conference on Computer Vision (ICCV)}.

\bibitem[{Chung et~al.(2023)Chung, Kim, Mccann, Klasky, and Ye}]{chungdiffusion}
Chung, H.; Kim, J.; Mccann, M.~T.; Klasky, M.~L.; and Ye, J.~C. 2023.
\newblock Diffusion Posterior Sampling for General Noisy Inverse Problems.
\newblock In \emph{The Eleventh International Conference on Learning Representations}.

\bibitem[{Chung et~al.(2022)Chung, Sim, Ryu, and Ye}]{chung2022improving}
Chung, H.; Sim, B.; Ryu, D.; and Ye, J.~C. 2022.
\newblock Improving Diffusion Models for Inverse Problems using Manifold Constraints.
\newblock \emph{arXiv preprint arXiv:2206.00941}.

\bibitem[{Chung, Sim, and Ye(2022)}]{chung2022come}
Chung, H.; Sim, B.; and Ye, J.~C. 2022.
\newblock {Come-Closer-Diffuse-Faster: Accelerating Conditional Diffusion Models for Inverse Problems through Stochastic Contraction}.
\newblock In \emph{Proceedings of the IEEE/CVF Conference on Computer Vision and Pattern Recognition}.

\bibitem[{Deng et~al.(2009)Deng, Dong, Socher, Li, Li, and Fei-Fei}]{deng2009imagenet}
Deng, J.; Dong, W.; Socher, R.; Li, L.-J.; Li, K.; and Fei-Fei, L. 2009.
\newblock Imagenet: A large-scale hierarchical image database.
\newblock In \emph{2009 IEEE conference on computer vision and pattern recognition}, 248--255. Ieee.

\bibitem[{Dhariwal and Nichol(2021)}]{dhariwal2021diffusion}
Dhariwal, P.; and Nichol, A.~Q. 2021.
\newblock Diffusion Models Beat {GAN}s on Image Synthesis.
\newblock In Beygelzimer, A.; Dauphin, Y.; Liang, P.; and Vaughan, J.~W., eds., \emph{Advances in Neural Information Processing Systems}.

\bibitem[{Dou and Song(2024)}]{dou2024diffusion}
Dou, Z.; and Song, Y. 2024.
\newblock Diffusion Posterior Sampling for Linear Inverse Problem Solving: A Filtering Perspective.
\newblock In \emph{The Twelfth International Conference on Learning Representations}.

\bibitem[{Efron(2011)}]{efron2011tweedie}
Efron, B. 2011.
\newblock Tweedie’s formula and selection bias.
\newblock \emph{Journal of the American Statistical Association}, 106(496): 1602--1614.

\bibitem[{Feng and Bouman(2023)}]{feng2023efficient}
Feng, B.~T.; and Bouman, K.~L. 2023.
\newblock Efficient Bayesian Computational Imaging with a Surrogate Score-Based Prior.
\newblock \emph{arXiv preprint arXiv:2309.01949}.

\bibitem[{Heusel et~al.(2017)Heusel, Ramsauer, Unterthiner, Nessler, and Hochreiter}]{heusel2017gans}
Heusel, M.; Ramsauer, H.; Unterthiner, T.; Nessler, B.; and Hochreiter, S. 2017.
\newblock Gans trained by a two time-scale update rule converge to a local nash equilibrium.
\newblock \emph{Advances in neural information processing systems}, 30.

\bibitem[{Ho, Jain, and Abbeel(2020)}]{ho2020denoising}
Ho, J.; Jain, A.; and Abbeel, P. 2020.
\newblock Denoising Diffusion Probabilistic Models.
\newblock In \emph{Advances in Neural Information Processing Systems}, volume~33, 6840--6851.

\bibitem[{Ho et~al.(2022)Ho, Salimans, Gritsenko, Chan, Norouzi, and Fleet}]{ho2022video}
Ho, J.; Salimans, T.; Gritsenko, A.; Chan, W.; Norouzi, M.; and Fleet, D.~J. 2022.
\newblock Video Diffusion Models.
\newblock \emph{arXiv preprint arXiv:2204.03458}.

\bibitem[{Jalal et~al.(2021{\natexlab{a}})Jalal, Arvinte, Daras, Price, Dimakis, and Tamir}]{jalal2021robust}
Jalal, A.; Arvinte, M.; Daras, G.; Price, E.; Dimakis, A.~G.; and Tamir, J. 2021{\natexlab{a}}.
\newblock Robust compressed sensing mri with deep generative priors.
\newblock \emph{Advances in Neural Information Processing Systems}, 34: 14938--14954.

\bibitem[{Jalal et~al.(2021{\natexlab{b}})Jalal, Karmalkar, Hoffmann, Dimakis, and Price}]{jalal2021fairness}
Jalal, A.; Karmalkar, S.; Hoffmann, J.; Dimakis, A.; and Price, E. 2021{\natexlab{b}}.
\newblock Fairness for image generation with uncertain sensitive attributes.
\newblock In \emph{International Conference on Machine Learning}, 4721--4732. PMLR.

\bibitem[{Kadkhodaie and Simoncelli(2021)}]{kadkhodaie2021stochastic}
Kadkhodaie, Z.; and Simoncelli, E. 2021.
\newblock Stochastic solutions for linear inverse problems using the prior implicit in a denoiser.
\newblock \emph{Advances in Neural Information Processing Systems}, 34: 13242--13254.

\bibitem[{Karras, Laine, and Aila(2019)}]{karras2019style}
Karras, T.; Laine, S.; and Aila, T. 2019.
\newblock A style-based generator architecture for generative adversarial networks.
\newblock In \emph{Proceedings of the IEEE/CVF conference on computer vision and pattern recognition}, 4401--4410.

\bibitem[{Kawar et~al.(2022)Kawar, Elad, Ermon, and Song}]{kawar2022denoising}
Kawar, B.; Elad, M.; Ermon, S.; and Song, J. 2022.
\newblock Denoising diffusion restoration models.
\newblock \emph{Advances in Neural Information Processing Systems}, 35: 23593--23606.

\bibitem[{Kawar, Vaksman, and Elad(2021)}]{kawar2021snips}
Kawar, B.; Vaksman, G.; and Elad, M. 2021.
\newblock SNIPS: Solving noisy inverse problems stochastically.
\newblock \emph{Advances in Neural Information Processing Systems}, 34: 21757--21769.

\bibitem[{Kolouri et~al.(2019)Kolouri, Nadjahi, Simsekli, Badeau, and Rohde}]{kolouri2019generalized}
Kolouri, S.; Nadjahi, K.; Simsekli, U.; Badeau, R.; and Rohde, G. 2019.
\newblock Generalized sliced wasserstein distances.
\newblock \emph{Advances in neural information processing systems}, 32.

\bibitem[{Kong et~al.()Kong, Ping, Huang, Zhao, and Catanzaro}]{kongdiffwave}
Kong, Z.; Ping, W.; Huang, J.; Zhao, K.; and Catanzaro, B. ????
\newblock DiffWave: A Versatile Diffusion Model for Audio Synthesis.
\newblock In \emph{International Conference on Learning Representations}.

\bibitem[{Lu et~al.(2022)Lu, Zheng, Bao, Chen, Li, and Zhu}]{lu2022maximum}
Lu, C.; Zheng, K.; Bao, F.; Chen, J.; Li, C.; and Zhu, J. 2022.
\newblock Maximum likelihood training for score-based diffusion odes by high order denoising score matching.
\newblock In \emph{International Conference on Machine Learning}, 14429--14460. PMLR.

\bibitem[{Mardani et~al.(2023)Mardani, Song, Kautz, and Vahdat}]{mardani2023variational}
Mardani, M.; Song, J.; Kautz, J.; and Vahdat, A. 2023.
\newblock A Variational Perspective on Solving Inverse Problems with Diffusion Models.
\newblock \emph{arXiv preprint arXiv:2305.04391}.

\bibitem[{Meng et~al.(2021)Meng, Song, Li, and Ermon}]{meng2021estimating}
Meng, C.; Song, Y.; Li, W.; and Ermon, S. 2021.
\newblock Estimating High Order Gradients of the Data Distribution by Denoising.
\newblock In Beygelzimer, A.; Dauphin, Y.; Liang, P.; and Vaughan, J.~W., eds., \emph{Advances in Neural Information Processing Systems}.

\bibitem[{Nichol et~al.(2022)Nichol, Dhariwal, Ramesh, Shyam, Mishkin, Mcgrew, Sutskever, and Chen}]{nichol2022glide}
Nichol, A.~Q.; Dhariwal, P.; Ramesh, A.; Shyam, P.; Mishkin, P.; Mcgrew, B.; Sutskever, I.; and Chen, M. 2022.
\newblock GLIDE: Towards Photorealistic Image Generation and Editing with Text-Guided Diffusion Models.
\newblock In \emph{International Conference on Machine Learning}, 16784--16804. PMLR.

\bibitem[{Peng et~al.(2024)Peng, Zheng, Dai, Xiao, Li, Zou, and Xiong}]{peng2024improving}
Peng, X.; Zheng, Z.; Dai, W.; Xiao, N.; Li, C.; Zou, J.; and Xiong, H. 2024.
\newblock Improving Diffusion Models for Inverse Problems Using Optimal Posterior Covariance.
\newblock In \emph{Forty-first International Conference on Machine Learning}.

\bibitem[{Saharia et~al.(2022)Saharia, Chan, Saxena, Li, Whang, Denton, Ghasemipour, Gontijo~Lopes, Karagol~Ayan, Salimans et~al.}]{saharia2022photorealistic}
Saharia, C.; Chan, W.; Saxena, S.; Li, L.; Whang, J.; Denton, E.~L.; Ghasemipour, K.; Gontijo~Lopes, R.; Karagol~Ayan, B.; Salimans, T.; et~al. 2022.
\newblock Photorealistic text-to-image diffusion models with deep language understanding.
\newblock \emph{Advances in neural information processing systems}, 35: 36479--36494.

\bibitem[{Song, Meng, and Ermon(2021)}]{song2020denoising}
Song, J.; Meng, C.; and Ermon, S. 2021.
\newblock Denoising Diffusion Implicit Models.
\newblock In \emph{9th International Conference on Learning Representations, {ICLR}}.

\bibitem[{Song et~al.(2023)Song, Vahdat, Mardani, and Kautz}]{song2023pseudoinverse}
Song, J.; Vahdat, A.; Mardani, M.; and Kautz, J. 2023.
\newblock Pseudoinverse-guided diffusion models for inverse problems.
\newblock In \emph{International Conference on Learning Representations}.

\bibitem[{Song et~al.(2022)Song, Shen, Xing, and Ermon}]{song2022solving}
Song, Y.; Shen, L.; Xing, L.; and Ermon, S. 2022.
\newblock Solving Inverse Problems in Medical Imaging with Score-Based Generative Models.
\newblock In \emph{International Conference on Learning Representations}.

\bibitem[{Song et~al.(2021)Song, Sohl{-}Dickstein, Kingma, Kumar, Ermon, and Poole}]{song2020score}
Song, Y.; Sohl{-}Dickstein, J.; Kingma, D.~P.; Kumar, A.; Ermon, S.; and Poole, B. 2021.
\newblock Score-Based Generative Modeling through Stochastic Differential Equations.
\newblock In \emph{9th International Conference on Learning Representations, {ICLR}}.

\bibitem[{Stevens et~al.(2023)Stevens, van Gorp, Meral, Shin, Yu, Robert, and van Sloun}]{stevens2023removing}
Stevens, T.~S.; van Gorp, H.; Meral, F.~C.; Shin, J.; Yu, J.; Robert, J.-L.; and van Sloun, R.~J. 2023.
\newblock Removing structured noise with diffusion models.
\newblock \emph{arXiv preprint arXiv:2302.05290}.

\bibitem[{Vincent(2011)}]{vincent2011connection}
Vincent, P. 2011.
\newblock A connection between score matching and denoising autoencoders.
\newblock \emph{Neural computation}, 23(7): 1661--1674.

\bibitem[{Zhang et~al.(2023)Zhang, Zhang, Zhang, and Kweon}]{zhang2023text}
Zhang, C.; Zhang, C.; Zhang, M.; and Kweon, I.~S. 2023.
\newblock Text-to-image diffusion models in generative ai: A survey.
\newblock \emph{arXiv preprint arXiv:2303.07909}.

\bibitem[{Zhang et~al.(2018)Zhang, Isola, Efros, Shechtman, and Wang}]{zhang2018unreasonable}
Zhang, R.; Isola, P.; Efros, A.~A.; Shechtman, E.; and Wang, O. 2018.
\newblock The unreasonable effectiveness of deep features as a perceptual metric.
\newblock In \emph{Proceedings of the IEEE conference on computer vision and pattern recognition}, 586--595.

\end{thebibliography}

\onecolumn
\section{Anonymous Repo of the Project}
Please use the following website for the code and the repo of the paper:

\url{https://anonymous.4open.science/r/AAAI-2025-Covariance-Aware-Diffusion-Models-80E8/README.md}

\section{Proof of Theorem 1}
\begin{proof}
The marginal distribution $p({\boldsymbol{y}})$ could be expressed as
\begin{align}
    p({\boldsymbol{y}}) & = \int p({\boldsymbol{y}}|{\boldsymbol{\eta}}) p({\boldsymbol{\eta}}) d{\boldsymbol{\eta}}\\
    & = \int p_0({\boldsymbol{y}}) \exp{\left({\boldsymbol{\eta}}^\top T({\boldsymbol{y}})-\varphi({\boldsymbol{\eta}})\right)}p({\boldsymbol{\eta}}) d{\boldsymbol{\eta}}.
\end{align}
Then, the derivative of the marginal distribution $p({\boldsymbol{y}})$ with respect to $y$ becomes
\begin{align}
    \nabla_{y_i} p({\boldsymbol{y}}) & = {\nabla_{y_i} p_0({\boldsymbol{y}})}\int \exp{\left({\boldsymbol{\eta}}^\top T({\boldsymbol{y}})-\varphi({\boldsymbol{\eta}})\right)}p({\boldsymbol{\eta}})d{\boldsymbol{\eta}} +
                \int (\nabla_{y_i} T({\boldsymbol{y}}))^\top {\boldsymbol{\eta}} p_0({\boldsymbol{y}})\exp{\left({\boldsymbol{\eta}}^\top T({\boldsymbol{y}})-\varphi({\boldsymbol{\eta}})\right)}p({\boldsymbol{\eta}})d{\boldsymbol{\eta}} \label{eq:firstder}\\
    & = \frac{\nabla_{y_i} p_0({\boldsymbol{y}})}{p_0({\boldsymbol{y}})}\int p({\boldsymbol{y}}|{\boldsymbol{\eta}}) p({\boldsymbol{\eta}}) d{\boldsymbol{\eta}}
     +  (\nabla_{y_i} T({\boldsymbol{y}}))^\top \int {\boldsymbol{\eta}} p({\boldsymbol{y}}|{\boldsymbol{\eta}}) p({\boldsymbol{\eta}}) d{\boldsymbol{\eta}}\\
       & = \frac{\nabla_{y_i} p_0({\boldsymbol{y}})}{p_0({\boldsymbol{y}})} p({\boldsymbol{y}})
     +  (\nabla_{y_i} T({\boldsymbol{y}}))^\top  \int {\boldsymbol{\eta}} p({\boldsymbol{y}},{\boldsymbol{\eta}}) d{\boldsymbol{\eta}}\\  
\end{align}
Therefore,
\begin{align}
    \frac{\nabla_{\boldsymbol{y}} p({\boldsymbol{y}})}{p({\boldsymbol{y}})} = \frac{\nabla_{\boldsymbol{y}} p_0({\boldsymbol{y}})}{p_0({\boldsymbol{y}})} + (\nabla_{\boldsymbol{y}} T({\boldsymbol{y}}))^\top  \int {\boldsymbol{\eta}} p({\boldsymbol{\eta}}|\boldsymbol{y})  d{\boldsymbol{\eta}}
\end{align}
which is equivalent to 
\begin{align}
(\nabla_{\boldsymbol{y}} T({\boldsymbol{y}}))^\top \mathbb{E}[{\boldsymbol{\eta}}|\boldsymbol{y}] = \nabla_{\boldsymbol{y}} \log p({\boldsymbol{y}})-\nabla_{\boldsymbol{y}}\log p_0({\boldsymbol{y}})
\end{align}

Now, we take another derivative w.r.t. $\boldsymbol{y}[j]$ from both sides of \cref{eq:firstder}: 
\begin{align}
\nabla_{y_j} \nabla_{y_i} p({\boldsymbol{y}}) & = \nabla_{y_j} \nabla_{y_i} p_0({\boldsymbol{y}}) \int \exp{\left({\boldsymbol{\eta}}^\top T({\boldsymbol{y}})-\varphi({\boldsymbol{\eta}})\right)}p({\boldsymbol{\eta}})d{\boldsymbol{\eta}} \nonumber \\
& \qquad+ \nabla_{y_i} p_0({\boldsymbol{y}}) (\nabla_{y_j} T({\boldsymbol{y}}))^\top
            \int  {\boldsymbol{\eta}} \exp{\left({\boldsymbol{\eta}}^\top T({\boldsymbol{y}})-\varphi({\boldsymbol{\eta}})\right)}p({\boldsymbol{\eta}})d{\boldsymbol{\eta}} \nonumber\\
& \qquad + \Big(\nabla_{y_j} \nabla_{y_i} T({\boldsymbol{y}}) \Big)^\top \odot p_0({\boldsymbol{y}}) \int  {\boldsymbol{\eta}}\exp{\left({\boldsymbol{\eta}}^\top T({\boldsymbol{y}})-\varphi({\boldsymbol{\eta}})\right)}p({\boldsymbol{\eta}})d{\boldsymbol{\eta}} \nonumber\\
& \qquad + \nabla_{y_j} p_0({\boldsymbol{y}}) (\nabla_{y_i} T({\boldsymbol{y}}))^\top \int  {\boldsymbol{\eta}}\exp{\left({\boldsymbol{\eta}}^\top T({\boldsymbol{y}})-\varphi({\boldsymbol{\eta}})\right)}p({\boldsymbol{\eta}})d{\boldsymbol{\eta}} \nonumber \\
& \qquad + (\nabla_{y_i} T({\boldsymbol{y}}))^\top p_0({\boldsymbol{y}})  \int  {\boldsymbol{\eta}} {\boldsymbol{\eta}}^\top \nabla_{y_j} T({\boldsymbol{y}}) \exp{\left({\boldsymbol{\eta}}^\top T({\boldsymbol{y}})-\varphi({\boldsymbol{\eta}})\right)}p({\boldsymbol{\eta}})d{\boldsymbol{\eta}} \\
& = \frac{\nabla_{y_j} \nabla_{y_i} p_0({\boldsymbol{y}}) }{p_0({\boldsymbol{y}})} p({\boldsymbol{y}}) + \frac{\nabla_{y_i} p_0({\boldsymbol{y}}) }{p_0({\boldsymbol{y}})} (\nabla_{y_j} T({\boldsymbol{y}}))^\top \int {\boldsymbol{\eta}} p({\boldsymbol{y}},{\boldsymbol{\eta}}) d{\boldsymbol{\eta}} \nonumber\\
& \qquad + \Big(\big(\nabla_{y_j} \nabla_{y_i} T({\boldsymbol{y}})\big)^\top \odot p_0({\boldsymbol{y}}) + \nabla_{y_j} p_0({\boldsymbol{y}}) (\nabla_{y_i} T({\boldsymbol{y}}))^\top \Big) \frac{\int {\boldsymbol{\eta}} p({\boldsymbol{y}},{\boldsymbol{\eta}}) d{\boldsymbol{\eta}}}{p_0({\boldsymbol{y}})} \nonumber \\
& \qquad + (\nabla_{y_i} T({\boldsymbol{y}}))^\top  p_0({\boldsymbol{y}})  \int {\boldsymbol{\eta}} {\boldsymbol{\eta}}^\top \nabla_{y_j} T({\boldsymbol{y}})  p({\boldsymbol{y}},{\boldsymbol{\eta}}) d{\boldsymbol{\eta}} \label{eq:seclast}
\end{align}
Now, we divide both sides of \cref{eq:seclast} by  $p({\boldsymbol{y}})$ to get
\begin{align}
\frac{\nabla_{y_j} \nabla_{y_i} p({\boldsymbol{y}})}{p({\boldsymbol{y}})} & =    \frac{\nabla_{y_j} \nabla_{y_i} p_0({\boldsymbol{y}}) }{p_0({\boldsymbol{y}})} + \frac{\nabla_{y_i} p_0({\boldsymbol{y}}) }{p_0({\boldsymbol{y}})} (\nabla_{y_j} T({\boldsymbol{y}}))^\top \int {\boldsymbol{\eta}} p({\boldsymbol{\eta}} | \boldsymbol{y}) d{\boldsymbol{\eta}} \nonumber \\
& \qquad + \big(\nabla_{y_j} \nabla_{y_i} T({\boldsymbol{y}})\big)^\top \odot \int {\boldsymbol{\eta}} p({\boldsymbol{\eta}} | \boldsymbol{y}) d{\boldsymbol{\eta}} + \frac{ \nabla_{y_j} p_0({\boldsymbol{y}}) (\nabla_{y_i} T({\boldsymbol{y}}))^\top}{p_0({\boldsymbol{y}})} \int {\boldsymbol{\eta}} p({\boldsymbol{\eta}} | \boldsymbol{y}) d{\boldsymbol{\eta}} \nonumber \\
& \qquad + (\nabla_{y_i} T({\boldsymbol{y}}))^\top  p_0({\boldsymbol{y}})  \int {\boldsymbol{\eta}} {\boldsymbol{\eta}}^\top \nabla_{y_j} T({\boldsymbol{y}}) p({\boldsymbol{\eta}} | \boldsymbol{y}) d{\boldsymbol{\eta}} \\
& =    \frac{\nabla_{y_j} \nabla_{y_i} p_0({\boldsymbol{y}}) }{p_0({\boldsymbol{y}})} + \frac{\nabla_{y_i} p_0({\boldsymbol{y}}) (\nabla_{y_j} T({\boldsymbol{y}}))^\top}{p_0({\boldsymbol{y}})} \mathbb{E}({\boldsymbol{\eta}}|\boldsymbol{y})  \nonumber \\
& \qquad + \big(\nabla_{y_j} \nabla_{y_i} T({\boldsymbol{y}})\big)^\top \odot \mathbb{E}({\boldsymbol{\eta}}|\boldsymbol{y}) + \frac{ \nabla_{y_j} p_0({\boldsymbol{y}}) (\nabla_{y_i} T({\boldsymbol{y}}))^\top}{p_0({\boldsymbol{y}})} \mathbb{E}({\boldsymbol{\eta}}|\boldsymbol{y}) \nonumber \\
& \qquad +(\nabla_{y_i} T({\boldsymbol{y}}))^\top   \mathbb{E}({\boldsymbol{\eta}} {\boldsymbol{\eta}}^\top |\boldsymbol{y}) \nabla_{y_j} T({\boldsymbol{y}})  \label{eq:last}
\end{align}
By isolating the term $(\nabla_{y_i} T({\boldsymbol{y}}))^\top   \mathbb{E}({\boldsymbol{\eta}} {\boldsymbol{\eta}}^\top |\boldsymbol{y}) \nabla_{y_j} T({\boldsymbol{y}})$ in \cref{eq:last}, and subtracting $ (\nabla_{y_i} T({\boldsymbol{y}}))^\top   \mathbb{E}({\boldsymbol{\eta}} |\boldsymbol{y}) \mathbb{E}({\boldsymbol{\eta}}^\top |\boldsymbol{y}) \nabla_{y_j} T({\boldsymbol{y}})$ from it, we obtain 
\begin{align}
(\nabla_{y_i} T({\boldsymbol{y}}))^\top    & \text{Cov}({\boldsymbol{\eta}}|\boldsymbol{y}) \nabla_{y_j} T({\boldsymbol{y}})=  (\nabla_{y_i} T({\boldsymbol{y}}))^\top   \mathbb{E}({\boldsymbol{\eta}} {\boldsymbol{\eta}}^\top |\boldsymbol{y}) \nabla_{y_j} T({\boldsymbol{y}}) -   (\nabla_{y_i} T({\boldsymbol{y}}))^\top   \mathbb{E}({\boldsymbol{\eta}} |\boldsymbol{y}) \mathbb{E}({\boldsymbol{\eta}}^\top |\boldsymbol{y}) \nabla_{y_j} T({\boldsymbol{y}}) \\
& = \frac{\nabla_{y_j} \nabla_{y_i} p({\boldsymbol{y}})}{p({\boldsymbol{y}})} - \frac{\nabla_{y_j} \nabla_{y_i} p_0({\boldsymbol{y}}) }{p_0({\boldsymbol{y}})} - \frac{\nabla_{y_i} p_0({\boldsymbol{y}}) (\nabla_{y_j} T({\boldsymbol{y}}))^\top}{p_0({\boldsymbol{y}})} \mathbb{E}({\boldsymbol{\eta}}|\boldsymbol{y}) - \big(\nabla_{y_j} \nabla_{y_i} T({\boldsymbol{y}})\big)^\top \odot \mathbb{E}({\boldsymbol{\eta}}|\boldsymbol{y}) \\
& \qquad - \frac{ \nabla_{y_j} p_0({\boldsymbol{y}}) (\nabla_{y_i} T({\boldsymbol{y}}))^\top}{p_0({\boldsymbol{y}})} \mathbb{E}({\boldsymbol{\eta}}|\boldsymbol{y}) - (\nabla_{y_i} T({\boldsymbol{y}}))^\top   \mathbb{E}({\boldsymbol{\eta}} |\boldsymbol{y}) \mathbb{E}({\boldsymbol{\eta}}^\top |\boldsymbol{y}) \nabla_{y_j} T({\boldsymbol{y}}) \\
& = \frac{\nabla_{y_j} \nabla_{y_i} p({\boldsymbol{y}})}{p({\boldsymbol{y}})} - \frac{\nabla_{y_j} \nabla_{y_i} p_0({\boldsymbol{y}}) }{p_0({\boldsymbol{y}})} - \nabla_{y_i} \log p_0({\boldsymbol{y}}) \Big( \nabla_{y_j} \log p({\boldsymbol{y}}) - \nabla_{y_j} \log  p_0({\boldsymbol{y}})\Big) \\
& \qquad - \nabla_{y_j} \log p_0({\boldsymbol{y}}) \Big( \nabla_{y_i} \log p({\boldsymbol{y}}) - \nabla_{y_i} \log  p_0({\boldsymbol{y}})\Big) - \big(\nabla_{y_j} \nabla_{y_i} T({\boldsymbol{y}})\big)^\top \odot \mathbb{E}({\boldsymbol{\eta}}|\boldsymbol{y}) \\
& \qquad - \Big( \nabla_{y_i} \log p({\boldsymbol{y}}) - \nabla_{y_i} \log  p_0({\boldsymbol{y}}) \Big) \Big( \nabla_{y_j} \log p({\boldsymbol{y}}) - \nabla_{y_j} \log  p_0({\boldsymbol{y}}) \Big) \\
& = \frac{\nabla_{y_j} \nabla_{y_i} p({\boldsymbol{y}})}{p({\boldsymbol{y}})} - \frac{\nabla_{y_j} \nabla_{y_i} p_0({\boldsymbol{y}}) }{p_0({\boldsymbol{y}})} + \nabla_{y_i} \log p_0({\boldsymbol{y}}) \nabla_{y_j} \log  p_0({\boldsymbol{y}}) \\
& \qquad - \nabla_{y_i} \log p({\boldsymbol{y}}) \nabla_{y_j} \log  p({\boldsymbol{y}}) - \big(\nabla_{y_j} \nabla_{y_i} T({\boldsymbol{y}})\big)^\top \odot \mathbb{E}({\boldsymbol{\eta}}|\boldsymbol{y})  \\
& =  \nabla_{y_j} \nabla_{y_i} \log p({\boldsymbol{y}})  - \nabla_{y_j} \nabla_{y_i} \log p_0({\boldsymbol{y}}) - \big(\nabla_{y_j} \nabla_{y_i} T({\boldsymbol{y}})\big)^\top \odot \mathbb{E}({\boldsymbol{\eta}}|\boldsymbol{y}).
\end{align}
Thus
\begin{align}
(\nabla_{\boldsymbol{y}} T({\boldsymbol{y}}))^\top    & \text{Cov}({\boldsymbol{\eta}}|\boldsymbol{y}) \nabla_{\boldsymbol{y}} T({\boldsymbol{y}})=  \nabla^2_{\boldsymbol{y}} \log p({\boldsymbol{y}}) - \nabla^2_{\boldsymbol{y}} \log p_0({\boldsymbol{y}}) - \nabla^2_{\boldsymbol{y}} T({\boldsymbol{y}}) \odot \mathbb{E}({\boldsymbol{\eta}}|\boldsymbol{y}),
\end{align}
which concludes the proof.
\end{proof}    

\section{Toy dataset}\label{appendix:toy}
The generation of this dataset is inspired from \citet{boys2023tweedie}.

As explained earlier in the paper, we model \( p_{0}(\boldsymbol{x}_0) \) as a mixture of 25 Gaussian distributions. Each of these Gaussian components has a mean vector \( \bold{U}_{i,j} \) in \( \mathbb{R}^{d} \), defined as \( \bold{U}_{i,j} = (8i, 8j, \dots, 8i, 8j) \) for each pair \( (i, j) \) where \( i \) and \( j \) take values from the set \(\{-2, -1, 0, 1, 2\} \). All components have the same variance of 1. The unnormalized weight associated with each component is \( \omega_{i,j} = 1.0 \). Additionally, we have set the variance of the noise, \( \sigma_{\delta}^{2} \), to \( 10^{-4} \).

Recall that the distribution \( p_{t}(\boldsymbol{x}_{t}) \) can be expressed as an integral: \( p_{t}(\boldsymbol{x}_{t}) = \int p_{t|0}(\boldsymbol{x}_{t}|\boldsymbol{x}_{0}) p_{0}(\boldsymbol{x}_{0}) d\boldsymbol{x}_{0} \). Since \( p_{0}(\boldsymbol{x}_{0}) \) is a mixture of Gaussian distributions, \( p_{t}(\boldsymbol{x}_{t}) \) is also a mixture of Gaussians. The means of these Gaussians are given by \( \sqrt{\alpha_{t}} \bold{U}_{i, j} \), and each Gaussian has unit variance. By using automatic differentiation libraries, we can efficiently compute the gradient \( \nabla_{\boldsymbol{x}_{t}} \log p_{t}(\boldsymbol{x}_{t}) \).

We have set the parameters \( \beta_{\text{max}} = 500.0 \) and \( \beta_{\text{min}} = 0.1 \), and we use 1000 timesteps to discretize the time domain. For a given pair of dimensions and a chosen observation noise standard deviation \( (d, m, \sigma) \), the measurement model \( (\boldsymbol{y}, \bold{A}) \) is generated as follows:

\noindent $\bullet$ Matrix \( \bold{A} \): First, we sample a random matrix \( \tilde{\bold{A}} \) from a Gaussian distribution \( \mathcal{N}(\mathbf{0}_{m \times d}, \bold{I}_{m \times d}) \). We then compute its singular value decomposition (SVD), \( \tilde{\bold{A}} = \bold{U} \bold{S} \bold{V}^{\top} \). For each pair \( (i, j) \) in \(\{-2, -1, 0, 1, 2\}^{2} \), we draw a singular value \( s_{i, j} \) from a uniform distribution on the interval \([0, 1]\). Finally, we construct the matrix \( \bold{A} = \bold{U} \text{diag}(\{s_{i, j}\}_{(i, j) \in \{-2, -1, 0, 1, 2\}^{2}})\bold{V}^{\top} \).

\noindent $\bullet$ Observation vector \( \boldsymbol{y} \): Next, we sample a vector \( \boldsymbol{x}_{*} \) from the distribution \( p_{0} \). The observation vector \( \boldsymbol{y} \) is then obtained by applying the matrix \( \bold{A} \) to \( \boldsymbol{x}_{*} \) and adding Gaussian noise \( \boldsymbol{z} \), where \( \boldsymbol{z} \) is sampled from \( \mathcal{N}(\mathbf{0}, \sigma^{2}\bold{I}_{m}) \).

Once we have drawn both $\boldsymbol{x}_{*} \sim p_{0}$ and $(\boldsymbol{y}, \bold{A}, \sigma)$, the posterior can be exactly calculated using Bayes formula and gives a mixture of Gaussians with mixture components $c_{i, j}$ and associated weights $\tilde{\omega}_{i, j}$,
\begin{align}
c_{i, j} &:= \mathcal{N}(\bold{\Sigma}(\bold{A}^{\top} \boldsymbol{y} / \sigma^{2} + \bold{U}_{i, j}), \bold{\Sigma}),\\
\tilde{\omega}_{i} &:= \omega_{i} \mathcal{N}(\boldsymbol{y}; \bold{A} \bold{U}_{i,j}, \sigma_{\delta}^{2}\bold{I}_{d} + \bold{A} \bold{A}^{\top}),
\end{align}
where $\bold{\Sigma} = (\bold{I}_{d} + \frac{1}{\sigma_{\delta}^{2}} \bold{A}^{\top} \bold{A})^{-1}$.

\subsection{SW Distance Calculation}
To compare the posterior distribution estimated by each algorithm with the target posterior distribution, we use $10^{4}$ slices for the SW distance and compare 1000 samples of the true posterior distribution.

\cref{table:a1} and \cref{table:a2} indicate the 95\% confidence intervals obtained by considering 20 randomly selected measurement models ($\bold{A}$) for each setting ($d, m, \sigma$). 

\begin{table}[!t] \caption{Sliced Wasserstein for VE-DDPM.}
\label{table:a1}
\renewcommand{\baselinestretch}{1.0}
\renewcommand{\arraystretch}{1.0}
\setlength{\tabcolsep}{2.2pt}
\centering
\resizebox{0.6\columnwidth}{!}{\begin{tabular}{|ll|ccc|ccc|ccc|}
\cline{1-11}
 & & \multicolumn{3}{c}{$\sigma = 0.01$}  & \multicolumn{3}{c}{$\sigma = 0.1$}  & \multicolumn{3}{c}{$\sigma = 1.0$}  \\ \hline
$d$ & $m$ & CA-DPS                     &  $\Pi$GDM                 & DPS             & CA-DPS   & $\Pi$GDM                 & DPS             & CA-DPS                     & $\Pi$GDM                 & DPS             \\ \hline
$8$     & $1$     & \textbf{1.9 \tiny $\pm$ 0.5}                    & 2.6 \tiny $\pm$ 0.9          & 4.7 \tiny $\pm$ 1.5 & \textbf{1.4 \tiny $\pm$ 0.6}  & 2.2 \tiny $\pm$ 0.9          & 4.7 \tiny $\pm$ 1.6 & \textbf{1.2 \tiny $\pm$ 0.6} &  1.5 \tiny $\pm$ 0.4          & 5.2 \tiny $\pm$ 1.3 \\
$8$     & $2$     & \textbf{0.8 \tiny $\pm$ 0.4}           & 2.1 \tiny $\pm$ 1.0          & 1.8 \tiny $\pm$ 1.5 & \textbf{1.0 \tiny $\pm$ 0.4} & 1.6 \tiny $\pm$ 0.6          & 1.5 \tiny $\pm$ 0.9 & \textbf{1.0 \tiny $\pm$ 0.3}           & 2.3 \tiny $\pm$ 0.4          & 3.5 \tiny $\pm$ 1.2 \\
$8$     & $4$     & \textbf{0.4 \tiny $\pm$ 0.2}  & 3.8 \tiny $\pm$ 2.3          & 0.7 \tiny $\pm$ 0.6 & \textbf{0.2 \tiny $\pm$ 0.2}  & 3.8 \tiny $\pm$ 2.2          & 0.8 \tiny $\pm$ 0.6 & \textbf{0.7 \tiny $\pm$ 0.3} & 1.8 \tiny $\pm$ 0.3          & 2.5 \tiny $\pm$ 0.9 \\
$80$    & $1$     & \textbf{2.7 \tiny $\pm$ 0.7}                    & 3.2 \tiny $\pm$ 1.0          & 5.6 \tiny $\pm$ 1.8 & \textbf{2.4 \tiny $\pm$ 0.8} & 2.9 \tiny $\pm$ 0.8          & 5.1 \tiny $\pm$ 1.8 & \textbf{1.5 \tiny $\pm$ 0.7}                    & 1.6 \tiny $\pm$ 0.5          & 6.9 \tiny $\pm$ 1.8 \\
$80$    & $2$     & \textbf{1.1 \tiny $\pm$ 0.6}         & 2.8 \tiny $\pm$ 1.3          & 3.2 \tiny $\pm$ 1.9 & \textbf{1.3 \tiny $\pm$ 0.4} & 2.7 \tiny $\pm$ 1.2          & 3.1 \tiny $\pm$ 1.9 & \textbf{1.0 \tiny $\pm$ 0.3}      & 1.4 \tiny $\pm$ 0.2          & 3.9 \tiny $\pm$ 1.2 \\
$80$    & $4$     & \textbf{0.4 \tiny $\pm$ 0.2}          & 0.6 \tiny $\pm$ 0.4          & 1.2 \tiny $\pm$ 1.1 & \textbf{0.5 \tiny $\pm$ 0.3} & 0.6 \tiny $\pm$ 0.4          & 1.0 \tiny $\pm$ 1.1 & \textbf{0.9 \tiny $\pm$ 0.3}   & 0.9 \tiny $\pm$ 0.2          & 1.7 \tiny $\pm$ 0.6 \\
$800$   & $1$     & \textbf{3.1 \tiny $\pm$ 0.7}                    & 3.5 \tiny $\pm$ 1.1          & 5.8 \tiny $\pm$ 1.6 & \textbf{3.0 \tiny $\pm$ 0.5} & 3.3 \tiny $\pm$ 0.9          & 5.7 \tiny $\pm$ 1.6 & \textbf{1.4 \tiny $\pm$ 0.5}           & 2.0 \tiny $\pm$ 0.4          & 6.8 \tiny $\pm$ 1.0 \\
$800$   & $2$     & \textbf{1.5 \tiny $\pm$ 0.5}           & 3.1 \tiny $\pm$ 1.1          & 3.5 \tiny $\pm$ 1.7 & \textbf{1.2 \tiny $\pm$ 0.4}  & 2.7 \tiny $\pm$ 0.9          & 3.1 \tiny $\pm$ 1.4 & \textbf{1.3 \tiny $\pm$ 0.4}        & 2.0 \tiny $\pm$ 0.5          & 4.7 \tiny $\pm$ 1.3 \\
$800$   & $4$     & 0.5 \tiny $\pm$ 0.3          & \textbf{0.4 \tiny $\pm$ 0.2} & 1.4 \tiny $\pm$ 1.0 & \textbf{0.3 \tiny $\pm$ 0.2} & 0.4 \tiny $\pm$ 0.2 & 1.3 \tiny $\pm$ 0.9 & 0.9 \tiny $\pm$ 0.2& \textbf{0.7 \tiny $\pm$ 0.3} & 0.9 \tiny $\pm$ 0.4 \\ \hline
\end{tabular}}
\end{table}

\begin{table}[h]
\caption{Sliced Wasserstein for the GMM case for the reverse VE SDEs discretized with Euler-Maruyama.}
\label{table:a2}
\renewcommand{\baselinestretch}{1.0}
\renewcommand{\arraystretch}{1.0}
\setlength{\tabcolsep}{2.2pt}
\centering
\resizebox{0.6\columnwidth}{!}{
\begin{tabular}{|ll|ccc|ccc|ccc|}
\cline{1-11}
        &         & \multicolumn{3}{c}{$\sigma = 0.01$}  & \multicolumn{3}{c}{$\sigma = 0.1$} & \multicolumn{3}{c}{$\sigma = 1.0$}  \\ \hline
$d$ & $m$ & CA-DPS  & $\Pi$GDM & DPS & CA-DPS  & $\Pi$GDM & DPS & CA-DPS & $\Pi$GDM & DPS \\ \hline
$8$     & $1$     & \textbf{1.6} \tiny $\pm$ 0.4 & 1.5 \tiny $\pm$ 0.4 & 5.7 \tiny $\pm$ 2.2 & \textbf{1.3} \tiny $\pm$ 0.4 & 1.2 \tiny $\pm$ 0.4 & 5.6 \tiny $\pm$ 2.1 & \textbf{0.8} \tiny $\pm$ 0.3 & 0.9 \tiny $\pm$ 0.3 & 0.9 \tiny $\pm$ 0.3 \\
$8$     & $2$     & \textbf{0.6} \tiny $\pm$ 0.3 & 0.4 \tiny $\pm$ 0.3 & 6.2 \tiny $\pm$ 0.8 & \textbf{1.0} \tiny $\pm$ 0.4 & 0.5 \tiny $\pm$ 0.3 & 6.2 \tiny $\pm$ 2.4 & \textbf{0.8} \tiny $\pm$ 0.2 & 1.0 \tiny $\pm$ 0.3 & 1.2 \tiny $\pm$ 0.4 \\
$8$     & $4$     & \textbf{0.4} \tiny $\pm$ 0.2 & 0.1 \tiny $\pm$ 0.1 & - & \textbf{0.4} \tiny $\pm$ 0.2 & 0.1 \tiny $\pm$ 0.0 & 8.4 \tiny $\pm$ 3.1 & \textbf{0.7} \tiny $\pm$ 0.2 & 0.2 \tiny $\pm$ 0.1 & 0.3 \tiny $\pm$ 0.2 \\
$80$    & $1$     & \textbf{2.5} \tiny $\pm$ 0.7 & 2.9 \tiny $\pm$ 1.4 & 9.1 \tiny $\pm$ 1.3 & \textbf{2.1} \tiny $\pm$ 0.8 & 2.1 \tiny $\pm$ 1.1 & 4.7 \tiny $\pm$ 1.8 & \textbf{1.4} \tiny $\pm$ 0.7 & 1.8 \tiny $\pm$ 0.8 & 1.9 \tiny $\pm$ 0.9 \\
$80$    & $2$     & \textbf{1.2} \tiny $\pm$ 0.4 & 0.8 \tiny $\pm$ 0.7 & 2.2 \tiny $\pm$ 0.9 & \textbf{1.1} \tiny $\pm$ 0.5 & 0.8 \tiny $\pm$ 0.7 & 6.0 \tiny $\pm$ 2.1 & \textbf{1.3} \tiny $\pm$ 0.3 & 1.3 \tiny $\pm$ 0.5 & 1.5 \tiny $\pm$ 0.5 \\
$80$    & $4$     & \textbf{0.4} \tiny $\pm$ 0.1 & 0.1 \tiny $\pm$ 0.0 & - & \textbf{0.3} \tiny $\pm$ 0.2 & 0.1 \tiny $\pm$ 0.1 & 4.4 \tiny $\pm$ 1.6 & \textbf{0.8} \tiny $\pm$ 0.3 & 0.4 \tiny $\pm$ 0.2 & 0.5 \tiny $\pm$ 0.3 \\
$800$   & $1$     & \textbf{3.2} \tiny $\pm$ 0.6 & 3.2 \tiny $\pm$ 1.0 & 6.8 \tiny $\pm$ 1.2 & \textbf{2.8} \tiny $\pm$ 0.5 & 2.8 \tiny $\pm$ 0.7 & 6.4 \tiny $\pm$ 1.5 & \textbf{1.4} \tiny $\pm$ 0.4 & 1.3 \tiny $\pm$ 0.3 & 1.3 \tiny $\pm$ 0.3 \\
$800$   & $2$     & \textbf{1.4} \tiny $\pm$ 0.3 & 0.8 \tiny $\pm$ 0.5 & 7.4 \tiny $\pm$ 0.9 & \textbf{1.2} \tiny $\pm$ 0.3 & 0.8 \tiny $\pm$ 0.4 & 6.4 \tiny $\pm$ 1.9 & \textbf{1.3} \tiny $\pm$ 0.4 & 1.1 \tiny $\pm$ 0.3 & 1.1 \tiny $\pm$ 0.3 \\
$800$   & $4$     & \textbf{0.4} \tiny $\pm$ 0.2 & 0.6 \tiny $\pm$ 0.5 & - & \textbf{0.3} \tiny $\pm$ 0.2 & 0.1 \tiny $\pm$ 0.0 & 5.8 \tiny $\pm$ 1.4 & \textbf{0.8} \tiny $\pm$ 0.3 & 0.4 \tiny $\pm$ 0.2 & 0.4 \tiny $\pm$ 0.2 \\ \bottomrule
\end{tabular}}
\end{table}

\section{More Qualitative Results}
In this section, we depict more reconstructed images using CA-DPS and compare it with those reconstructed by DPS. To this end, we pick 9 images from FFHQ dataset, and conduct super-resolution task ($16 \times$) with a Gaussian noise whose standard deviation is $\sigma = 0.05$. The results are depicted in \cref{fig:fig2}.

Furthermore, to visualize the reconstruction process over 1000 timesteps, we select a single image and display the reconstructed images throughout the denoising process, as illustrated in \cref{fig:fig3}.

\begin{figure*}[h]
\centering  \includegraphics[width=1\textwidth]{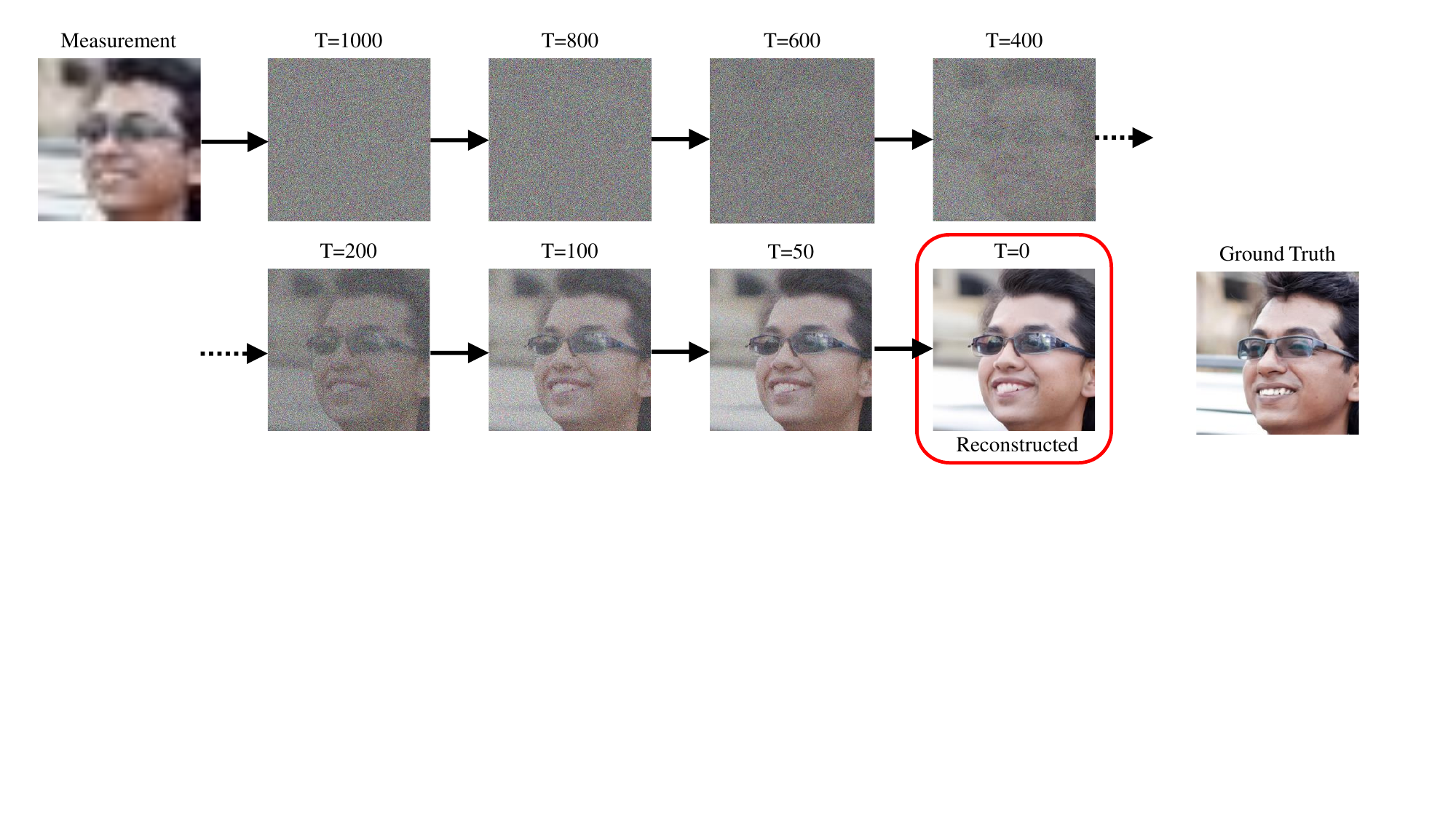}
\vspace{-2mm}
  \caption{Qualitative results on FFHQ dataset.
} \label{fig:fig3}
  \vspace{-2mm}
\end{figure*}

\begin{figure*}[!t]
\centering  \includegraphics[width=0.6\textwidth]{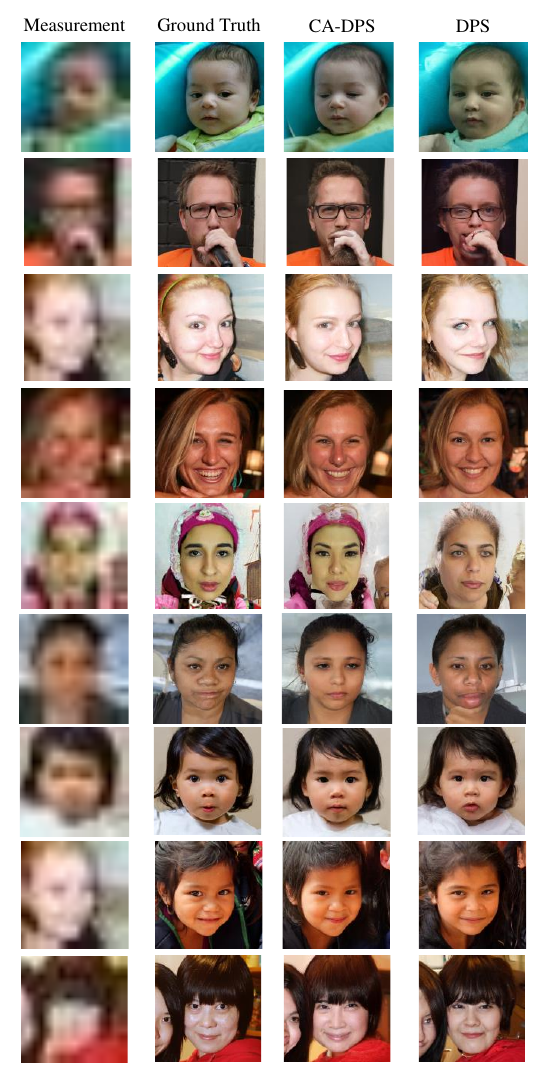}
\vspace{-2mm}
  \caption{Qualitative results on FFHQ dataset.
} \label{fig:fig2}
  \vspace{-2mm}
\end{figure*}

\end{document}